\documentclass{article}
\usepackage[margin=1in]{geometry}
\usepackage[authoryear,round]{natbib}

\usepackage{hyperref}
\usepackage{url}

\usepackage{amsmath, amssymb, amsthm}
\usepackage{xcolor}
\usepackage[capitalize]{cleveref}
\usepackage{graphicx}
\usepackage{subcaption}
\usepackage{color}
\usepackage{float}
\usepackage{booktabs}
\usepackage{multirow}
\usepackage{keytheorems}

\newkeytheorem{lemma,theorem,proposition,remark,assumption,definition,corollary}

\crefname{assumption}{Assumption}{Assumptions}
\crefname{definition}{Definition}{Definitions}

\title{A Sharp Transition in Data Reconstruction \\ under Differential Privacy}
\author{
  Max Cairney-Leeming \\
  \small Institute of Science and Technology Austria \\
  \small \texttt{max.cairney-leeming@ist.ac.at}
  \and
  Simone Bombari \\
  \small Institute of Science and Technology Austria \\
  \small \texttt{simone.bombari@ist.ac.at}
  \and
  Marco Mondelli \\
  \small Institute of Science and Technology Austria \\
  \small \texttt{marco.mondelli@ist.ac.at}
}
\date{}

\def\epsilon{\varepsilon}

\DeclareMathOperator{\tr}{tr}
\newcommand{\opnorm}[1]{\left\lVert #1 \right\rVert_{\mathrm{op}}}
\newcommand{\evmin}[1]{\lambda_{\min}\!\left(#1\right)}
\newcommand{\evmax}[1]{\lambda_{\max}\!\left(#1\right)}
\newcommand{\svmin}[1]{s_{\min}\!\left(#1\right)}
\newcommand{\svmax}[1]{s_{\max}\!\left(#1\right)}

\def\x{\mathcal{X}}%

\newcommand{\reals}{\mathbb{R}}

\newcommand{\N}{\mathcal{N}}%

\newcommand{\ex}{\mathbb{E}}

\newcommand{\prob}{\mathbb{P}}

\newcommand{\argmin}{\operatorname*{arg\;min}}

\newcommand{\sign}{\operatorname{sign}}

\newcommand{\n}{\|}
\newcommand{\norm}{\|\cdot\|}%
\newcommand{\op}{\mathrm{op}}

\usepackage{xspace}
\makeatletter
\DeclareRobustCommand\onedot{\futurelet\@let@token\@onedot}
\def\@onedot{\ifx\@let@token.\else.\null\fi\xspace}

\newcommand{\iid}{{i.i.d}\onedot}
\newcommand{\eg}{{e.g}\onedot} 
\newcommand{\ie}{{i.e}\onedot}

\makeatother

\newcommand{\kl}{D_\mathrm{KL}}

\newcommand{\dd}{\;\xspace\mathrm{d}}

\begin{document}

\maketitle

\newcommand{\numPoints}{n}
\newcommand{\dimInputs}{d}

\newcommand{\fullPairsDS}{Z}
\newcommand{\restOfPairsDS}{Z_{-}}
\newcommand{\fullDS}{X}
\newcommand{\labels}{Y}
\newcommand{\restOfDS}{{X_-}}
\newcommand{\restOfLabels}{{y_-}}
\newcommand{\targetPoint}{x_t}
\newcommand{\targetLabel}{{y_t}}
\newcommand{\targetPair}{{z_t}}
\newcommand{\attackPoint}{\widehat{x}}
\newcommand{\reconstructionError}{%
  \min_{\tau\in[-1,1]}
\left\|\targetPoint-\tau\attackPoint\right\|_2^2}

\newcommand{\mech}{\mathcal{M}}
\newcommand{\attack}{\mathcal{A}}
\newcommand{\paramSpace}{\Theta}

\newcommand{\optWeights}{{\theta^*}}
\newcommand{\noisyWeights}{\tilde{\theta}} %
\newcommand{\dpNoise}{{b_\mathrm{DP}}}
\newcommand{\dpStdDev}{\sigma_{\mathrm{DP}}}
\newcommand{\restWeights}{{\theta_-}}
\newcommand{\trueWeights}{{\theta_\mathrm{true}}}

\newcommand{\labelNoise}{{\nu}} %

\newcommand{\labelStdDev}{{\zeta}} %

\newcommand{\featureRadiusBound}{{R}}
\newcommand{\clipConstant}{C_{\mathrm{clip}}}
\newcommand{\targetRadius}{{R_t}}
\newcommand{\huberLoss}{\ell_{\clipConstant}}

\newcommand{\huberLooGrad}{F_{-}}

\newcommand{\targetProj}{{x_{\parallel}}}
\newcommand{\targetProjRest}{{x_{\perp}}}
\newcommand{\attackProj}{{\widehat{x}_{\parallel}}}
\newcommand{\dataPar}{{X_{\parallel}}}
\newcommand{\dataPerp}{{X_{\perp}}}
\newcommand{\pcaDataPar}{{\widehat{X}_{\parallel}}}
\newcommand{\pcaDataPerp}{{\widehat{X}_{\perp}}}
\newcommand{\sampleFactor}{{\gamma_{n,s}}}
\newcommand{\pcaSampleFactor}{{\widehat{\gamma}_{n,s}}}
\newcommand{\bulkRatio}{{\tau_b}}

\newcommand{\pcaProjAttack}{v_\text{PCA}}
\newcommand{\eigengap}{g}
\newcommand{\pcaProj}{\widehat{P}}

\newcommand{\pcaPerpProj}{\widehat{Q}}
\newcommand{\Mdp}{M_{\text{dp}}}
\newcommand{\pcaError}{\delta_{\mathrm{PCA}}}

\newcommand{\channel}{\mathcal{C}}
\newcommand{\inputSpace}{\mathcal{Z}}
\newcommand{\rateChan}{\mathcal{R}}

\newcommand{\distChan}{\mathcal{D}}
\newcommand{\rateDist}{\mathfrak{r}}
\newcommand{\distRate}{\mathfrak{d}}

\begin{abstract}

Data reconstruction attacks have empirically been successful in
recovering training samples from learned
models, raising privacy concerns and motivating defenses with
guarantees that remain valid against future threats. While
differential privacy (DP) provides formal protection, choosing
the privacy budget remains a challenge: small budgets
severely reduce utility, but it is hard to quantify how large the
budget can be without allowing accurate
reconstruction.
In this work, we study informed attackers who aim to reconstruct a
single $d$-dimensional training sample from a $\rho$-zero-concentrated
DP model, knowing all other
training data. Our main contribution is to establish a sharp
transition at $\rho \asymp d$ for data reconstruction: on the one
hand, we derive entropy-based lower bounds for any private
mechanism and any attack,
characterizing a set of target priors for which reconstruction is
information-theoretically impossible for $\rho \ll d$; on the other
hand, we analyze a simple attack on private linear regression with
output perturbation, showing that reconstruction is practically
feasible for $\rho \gg d$. Remarkably, the transition moves to
$\rho \asymp s$ for data lying in an $s$-dimensional subspace,
demonstrating that the privacy
budget guaranteeing adequate protection must be assessed in terms of the effective dimension of the data. We 
validate our findings via experiments on synthetic data and natural images (CIFAR-10, ImageNet).

\end{abstract}

\section{Introduction}
\label{sec:intro}

Easy access to vast amounts of training data is a key ingredient in the success of modern deep learning, naturally encouraging the collection of data that may contain sensitive personal information or copyrighted material. Standard training methods often memorize such training data, making it possible for attackers to recover information about their content from the final model \citep{carlini2021, nasr2025scalable, cooper2025extracting}. For example, membership inference attacks are designed to determine whether a given target sample $x_t$ was included in the training set \citep{membershipinference, whitebox, carlini2022membership}, while \emph{data reconstruction} aims to recover the full $x_t$ from scratch \citep{haimReconstructingTrainingData2022a, buzaglo2023reconstructing}.

A way to address these risks is via \emph{differential privacy} (DP) \citep{dworkCalibratingNoiseSensitivity2006}, which has become the standard paradigm for provably safeguarding models from leaking private information \eg through unintended memorization. DP quantifies data protection through a numerical parameter ($\rho$ in \cref{def:zcdp}), which bounds the impact a single data point can have on the output of the algorithm. However, the noise required by DP training negatively affects performance \citep{Abadi2016}, thus limiting its adoption. %
In particular, this performance cost is generally noticeable for constant-order (i.e., not dependent on sample size or data dimensionality) privacy budgets $\rho \asymp 1$, a regime that is also the usual focus in the literature \citep{de2022, mckenna2025}\footnote{The cited work often focuses on $(\varepsilon, \delta)$-DP with $\varepsilon \asymp 1$. Given the conversion from $\rho$-zCDP to $(\varepsilon, \delta)$-DP \citep{bunConcentratedDifferentialPrivacy2016}, we interpret this regime as $\rho \asymp 1$, neglecting the logarithmic dependence on $\delta$. As an  example, a zCDP guarantee with $\rho \simeq 1.05$ implies the $(8, 10^{-5})$-DP guarantee reported by \citet{de2022}.}.
This scaling provides non-trivial trade-offs for membership inference attacks \citep{dong2022gaussian, mahloujifar2022optimal,pmlr-v37-kairouz15}: informally, attacks are possible when $\rho \gg 1$, and impossible when $\rho \ll 1$.

On the other hand, much larger privacy budgets yield negligible utility losses in some domains and still protect models from reconstruction attacks \citep{bhowmick2019protectionreconstructionapplicationsprivate, stock2022defendingreconstructionattacksrenyi, ziller2024reconciling}, motivating interest in the regime $\rho \gg 1$ \citep{cyffers2025setting}. %
In particular, \cite{balleReconstructingTrainingData2022} have shown that, for suitable priors on $x_t \in \mathbb{R}^d$, $\rho \ll d$ suffices to make reconstruction statistically impossible. However, for general priors, their bound relies on small-ball probabilities which may be hard to compute in practice (e.g., for natural language data). Most importantly, it is unclear whether the result is tight, in the sense that reconstruction attacks are possible for $\rho \gg d$.

Our work takes a step towards understanding when the privacy budget sharply characterizes the feasibility of reconstruction from the final model parameters, identifying a sharp transition at $\rho \asymp d$ in \emph{DP linear regression}. We focus on informed attackers with access to the training set (except for the target example), as well as the hyper-parameters of the \emph{output perturbation} algorithm used to train the model. More specifically, our contributions are the following:
\begin{enumerate}
\item We derive a lower bound on the expected reconstruction error in terms of the target's differential entropy $h(x_t)$, valid for any $\rho$-zCDP mechanism and any attack (\cref{theorem:rho-entropy-reconstruction}). This result implies that accurate reconstruction is impossible as long as $\rho \ll h(x_t)$. For the target distributions considered by \citet{balleReconstructingTrainingData2022}, or for distributions with independent entries, we have $h(x_t) \asymp d$, demonstrating the impossibility of reconstruction for $\rho \ll d$. %

\item We characterize the performance of an explicit reconstruction attack on a DP model trained via output perturbation (\cref{thm:direct-huber-reconstruction}). For a natural range of regularization parameters $\lambda$ and residual clipping thresholds $\clipConstant$, the attack succeeds with high probability whenever $\rho \gg d$, providing a tight converse to \cref{theorem:rho-entropy-reconstruction}.

\item We extend our results to data distributions with lower \emph{effective dimension}, in particular to data supported on an $s$-dimensional subspace of $\mathbb{R}^d$. We obtain an analogous entropy-based lower bound (\cref{theorem:low-dim-entropy-bound}) and derive a reconstruction attack that exploits knowledge of the subspace (\cref{thm:low-dim-huber-reconstruction}). Together, these results show that the relevant transition for data reconstruction in this setting occurs at $\rho \asymp s$.

\end{enumerate}

We validate our findings through numerical experiments on synthetic data (\cref{fig:combined-synth-viz,fig:pca-reconstruction-synthetic}) and natural images from CIFAR-10 and ImageNet (\cref{fig:imagenet,fig:visual-reconstruction,fig:pca-reconstruction-cifar10}),
demonstrating the wider generality of our results beyond the technical assumptions.
Overall, our work identifies data dimensionality as a relevant scale for assessing privacy budgets, providing a concrete basis for studying the risks and potential utility benefits of larger budgets.

\section{Related work}

\paragraph{Data reconstruction from non-private models.}
Large models are known to memorize individual training data within their %
parameters. Such data can be extracted through appropriate prompting in generative models \citep{carliniExtractDiffusion2023, nasr2025scalable, cooper2025extracting}, or \emph{reconstructed} directly from the parameters of classifiers \citep{haimReconstructingTrainingData2022a, buzaglo2023reconstructing, oz2024reconstructing,shenDataReconstructionIdentifiability2026}. In the theoretical literature, the problem of fully reconstructing the training set has been tackled in specific settings: the heuristic method of \citet{haimReconstructingTrainingData2022a} relies on the implicit bias of gradient descent in homogeneous networks \citep{lyu2020gradient, ji2020directional};  \citet{loo2024understanding} and \citet{iurada2025law} show that reconstruction is possible in, respectively, infinite-width shallow networks and sufficiently  over-parameterized models. %
If the attacker knows the rest of the dataset, an effective reconstruction method against a generalized linear model, trained without the clipping or noise required for privacy, has also been expressed in closed form by \cite {balleReconstructingTrainingData2022}.

\paragraph{Attacks on private models.}

For models trained with DP, the feasibility of membership inference attacks (MIAs) is closely linked to the privacy budget \citep{yeom2018privacy, pmlr-v37-kairouz15,mahloujifar2022optimal}, and there is a  %
connection to privacy auditing, as these results mean that successful membership attacks set lower bounds on the privacy guarantee of an algorithm
\citep{jagielski2020auditing, nasr2023tight, steinke2023privacy}. 
Reconstruction attacks have also been linked to privacy budgets, primarily through reconstruction robustness \citep{balleReconstructingTrainingData2022}. %
This framework has been extended to give stronger bounds for specific mechanisms \citep{hayesBoundingTrainingData2023,kaissis2023bounding}, to use $f$-DP \citep{kulynychUnifyingReIdentificationAttribute2025}, and to reconstruct multiple data points \citep{swanberg2026unifiedframeworkadversaryawaredifferential}. Additional related work has focused on recovering one batch from DP-SGD   \citep{liuDataReconstructionAttacks2025}, %
specializing reconstruction robustness to discrete alphabets \citep{guo2023fano}, or lower bounding the reconstruction error through the bias and variance of the %
attack \citep{guoBoundingTrainingData2022}, %
though requiring the privacy budget to be at most logarithmic in the diameter of the target prior to rule out reconstruction.
Our work combines impossibility guarantees based on the entropy of the target with a provably successful attack, thus establishing a sharp transition for data reconstruction.

\paragraph{Large privacy budgets.}
A classical convention in the DP literature is to take the privacy budget to be a moderately small constant \citep{Dwork2014}. %
This is consistent with theoretical analyses showing that MIAs fail when $\rho \ll 1$ \citep{bunConcentratedDifferentialPrivacy2016, pmlr-v37-kairouz15, dong2022gaussian}. When membership is not a sensitive attribute, much larger (possibly dimension-dependent) budgets provide meaningful protection while better preserving utility. This has been shown in the locally private setting \citep{bhowmick2019protectionreconstructionapplicationsprivate,sarwateRatedisortionPerspectiveLocal2014,kalantariRobustPrivacyUtilityTradeoffs2018}, in language models \citep{stock2022defendingreconstructionattacksrenyi}, and in medical imaging \citep{ziller2024reconciling}, with \citet{aerni2024evaluations} even finding that budgets well above $10^3$ can outperform other heuristic defenses against membership attacks. This evidence also partly motivates the recent position taken by \citet{cyffers2025setting}, who argues that larger budgets should not be dismissed in favor of uncertified protection methods. Our results provide a theoretical foundation for this perspective, identifying the data’s effective dimension as the scale for assessing how large privacy budgets can be while still protecting against accurate reconstruction.

\section{Preliminaries}
\label{sec:preliminaries}

\paragraph{Notation.}
We set $[n] := \{1,\ldots, n\}$. Given a vector $v$, let $\norm_2$ denote its
Euclidean norm. Given a matrix $A \in \reals^{n\times
m}$, let $\opnorm{A}$ be its operator (spectral) norm. Let
$h(a) = - \int p_a(x) \log p_a(x) \mathrm{d} x$ denote the
differential entropy of a random variable $a$.
We use the letters $x, \mathcal{X}$ to denote feature vectors and the
space of features;
$y, \mathcal{Y}$ for labels; and $z, \mathcal{Z}$ for feature-label
pairs.
Complexity notations $O$ ($\lesssim$), $\Omega$ ($\gtrsim$) and
$\Theta$ ($\asymp$) are meant for large input dimension
$d$ and sample size $n$.

\paragraph{Linear regression.}
Let $Z = (\fullDS,\labels)$ be a labeled training dataset, where
$\fullDS=[x_1,\ldots,x_{\numPoints}]^\top\in
\reals^{\numPoints\times\dimInputs}$ contains the training data
(features) on its
rows and $\labels=[y_1,\ldots,y_{\numPoints}]^\top\in
\reals^{\numPoints}$ contains the corresponding labels. We assume
features to have
bounded norm $\| x\|_2 \leq \featureRadiusBound$ and input-label pairs to be sampled \iid from a joint distribution
$P_{XY}$. We consider the \emph{linear regression} model
\begin{equation}\label{eq:datamodel}
  y_i=x_i^\top\trueWeights+\labelNoise_i,
\end{equation}
where $\trueWeights\in\reals^{\dimInputs}$ and
$\labelNoise_i$ is independent label noise.
The goal of DP linear regression is to output
an estimate that guarantees a privacy budget and
minimizes the \emph{test risk}, defined for
$\theta\in\reals^{\dimInputs}$ as
\begin{equation}\label{eq:P}
  \mathcal P(\theta)
  =\frac{1}{2}\ex_{(x,y)\sim P_{XY}}
  \left[\left(x^\top\theta-y\right)^2\right] 
\end{equation}
\paragraph{Differential privacy (DP) and output perturbation.}
We say that the datasets $\fullPairsDS$
and $\fullPairsDS'$ are \emph{adjacent} %
if they differ in a single input-label
pair. We quantify DP via $\rho$-zero-concentrated DP ($\rho$-zCDP)
which uses the $\alpha$-Rényi divergence between two probability
distributions $P, Q$: %
\begin{equation}
  D_\alpha(P\|Q)
  :=\frac{1}{\alpha-1}
  \log\int\left(\frac{\mathrm dP}{\mathrm dQ}\right)^\alpha
  \mathrm dQ,\qquad \alpha>1.
\end{equation}
\begin{definition}[$\rho$-zCDP
  \citep{bunConcentratedDifferentialPrivacy2016}]
  \label{def:zcdp}
  A randomized algorithm $\mech$ satisfies
  $\rho$-zCDP if, for every adjacent pair of datasets
  $\fullPairsDS, \fullPairsDS'$, and every $\alpha>1$,
  \begin{equation}
    D_\alpha\bigl(\mech(\fullPairsDS)\|\mech(\fullPairsDS')\bigr)
    \leq\rho\alpha.
  \end{equation}
\end{definition}

Fix a regularization parameter $\lambda>0$,  and a residual clipping threshold
$\clipConstant>0$.
Let $\huberLoss$ be the Huber loss, i.e., a square loss linearized
to be $\clipConstant$-Lipschitz:
\begin{equation}
  \huberLoss(a):=
  \begin{cases}
    a^2/2, & |a|\leq\clipConstant,\\
    \clipConstant|a|-\clipConstant^2/2,
    & |a|>\clipConstant.
  \end{cases}
\end{equation}
Writing $\fullPairsDS=(\fullDS,\labels)$, output perturbation first computes the
regularized empirical-risk minimizer
\begin{equation}\label{eq:clippedminimization}
  \optWeights(\fullPairsDS):=
  \argmin_{\theta\in\reals^{\dimInputs}}
  \left\{
    \frac{1}{\numPoints}\sum_{i=1}^{\numPoints}
    \huberLoss\!\left(
      x_i^\top\theta-y_i
    \right)
    +\frac{\lambda}{2}\|\theta\|_2^2
  \right\},
\end{equation}
and then releases
\begin{equation}\label{eq:output-perturbation}
  \noisyWeights=\optWeights(\fullPairsDS)+\dpNoise,
  \qquad
  \dpNoise\sim\N(0,\dpStdDev^2 I_{\dimInputs}).
\end{equation}

\begin{proposition}[store=prop:output-pert-is-private]
  \label{prop:output-pert-is-private}
  The output perturbation algorithm defined by
  \eqref{eq:clippedminimization}-\eqref{eq:output-perturbation} is
  $\rho$-zCDP provided
  \begin{equation}
    \label{eq:output-perturbation-noise-calibration}
    \dpStdDev^2
    =
    \frac{2\featureRadiusBound^2\clipConstant^2}
    {\lambda^2\rho\numPoints^2}.
  \end{equation}

\end{proposition}
This result follows standard arguments on the Gaussian mechanism
applied to minimization algorithms with bounded sensitivity
\citep{mironovRenyiDifferentialPrivacy2017}, and its proof is
deferred to \cref{subsec:output-pert-sens}. 
Intuitively, larger $\lambda$ (or smaller
$\clipConstant$) implies that the effect of a single training
sample on the resulting $\theta^*(Z)$ is smaller. Then, to guarantee
privacy, Gaussian noise is added proportionally to $1 / \rho$ and
$R^2 = \sup_x \| x \|_2^2$. We will refer to $\clipConstant$ and
$\lambda$ as the hyper-parameters of output perturbation. %

\paragraph{Reconstruction threat model.}

We follow the same threat model as
\cite{balleReconstructingTrainingData2022}. A trusted processor has a
dataset $\fullPairsDS=(
\fullDS,\labels)$, decomposed into a target point
$\targetPair=(\targetPoint,\targetLabel)$ and the rest of the
dataset, $\restOfPairsDS=(\restOfDS,\restOfLabels)$. For convenience,
we associate $t\in [n]$ with an index in the data. %
The processor releases the
output $\mech(\fullPairsDS) \in \paramSpace$ of a
$\rho$-zCDP mechanism on the dataset.
The attacker then develops a function $
\attack:\inputSpace^{\numPoints-1}\times\paramSpace \to\x$ to
reconstruct the target feature $\targetPoint \in \x$
from $\mech(\fullPairsDS) $ and all other datapoints
$\restOfPairsDS$. We note that the attacker also has knowledge of the
mechanism's hyperparameters $\lambda$ and $\clipConstant$. %
The resulting reconstruction is $ \attackPoint
:=\attack(\restOfPairsDS,\mech(\fullPairsDS))$.

\section{Main results}

\subsection{Information-theoretic lower bound}\label{sec:full_dim}

We first show a lower bound on the expected error
an attacker must incur when
reconstructing the target point, regardless of
$(\restOfDS,\restOfLabels)$, the $\rho$-zCDP mechanism, or the reconstruction method.

\begin{theorem}[store=rho-entropy-reconstruction]
  \label{theorem:rho-entropy-reconstruction}
  Let the target feature $\targetPoint$ be independent of the rest of
  the dataset $\restOfPairsDS=(\restOfDS,\restOfLabels)$. Let
  $\attack:\inputSpace^{\numPoints-1}\times\paramSpace\to\x$ be an
  attack, taking as input
  $\restOfPairsDS$ and the output
  $\mech(\restOfPairsDS\cup\{\targetPair\})$ of a
  $\rho$-zCDP mechanism, and producing as output an estimate
  $\attackPoint$ of $\targetPoint$.
  Then,
  \begin{equation}\label{eq:lower-bound-scale-invariant}
    \frac{\ex \n \targetPoint - \attackPoint \n_2^2}{\ex \n
    \targetPoint\n_2^2} \geq \frac{1}{2 \pi e}
    \exp\left(\frac{2}{d}\left(h(\targetPoint') -\rho\right)\right),
  \end{equation}
  with %
  $\targetPoint' :=
  \frac{\sqrt{d}}{\sqrt{\ex \n \targetPoint\n_2^2}} \targetPoint$.
\end{theorem}
We note that the bound in \eqref{eq:lower-bound-scale-invariant} is
scale-invariant,
in the sense that neither the LHS or RHS depend on the norm of $x_t$. Since
the LHS of \eqref{eq:lower-bound-scale-invariant} is $1$ for the
all-0 estimator and $0$ for $\hat x=x_t$, we conclude that any attack
fails whenever $\rho$ is small compared to the differential entropy
of the (normalized) point to
be reconstructed. This dependence on $h(\targetPoint')$ is natural:
the differential entropy quantifies the prior uncertainty in the target, with normalization removing its dependency on the scale of the distribution, and $\rho$ bounds the information that the mechanism can reveal.

When $h(x_t')=\Theta(d)$, Theorem
\ref{theorem:rho-entropy-reconstruction} implies that, if $\rho\ll
d$, reconstructing $\targetPoint$ accurately is
information-theoretically impossible. This is the case e.g.\ for isotropic Gaussians, for the uniform distribution on a ball, and for distributions with independent entries, each with marginal entropy bounded away from $0$. 
Moreover,
while \eqref{eq:lower-bound-scale-invariant} bounds normalized mean squared error in terms of $h(x_t')$,
similar bounds can be obtained for \emph{(i)} the per-coordinate error $\ex \n \targetPoint - \attackPoint \n_2^2 / d$ in terms of $h(\targetPoint)$, and for \emph{(ii)} different norms/orders (\eg, $\ex \n \targetPoint -
\attackPoint \n_2$) whenever the max-entropy distribution under
bounded distortion is computable. The proof of Theorem
\ref{theorem:rho-entropy-reconstruction} is in Appendix
\ref{sec:lower-bounds-entropy-mi}, with a sketch given below.

\paragraph{Proof sketch.} Let us consider normalized features
$\mathbb E [\| x_t \|_2^2] = d$ for simplicity, so that $x_t' = x_t$.
The idea is to combine an upper bound on
the mutual information coming from the privacy guarantee with a
rate-distortion inequality. First, in Lemma \ref{lem:zcdp-mut-info},
we show that $I(\targetPair; \mech(\restOfPairsDS \cup \{\targetPair\})) \leq \rho$.
This follows by taking the R\'enyi order $\alpha\downarrow1$ in the zCDP
condition and using convexity of the KL divergence to bound mutual information.
As shown by \cite{bunConcentratedDifferentialPrivacy2016}, an
application of the data processing inequality, combined with the
independence of $\targetPoint$ and $\restOfPairsDS$,
then implies
\begin{equation}\label{eq:sketch1}
  I(\targetPoint;\attackPoint)\leq\rho.
\end{equation}
Writing $D=\ex\|\targetPoint-\attackPoint\|_2^2$ and using that the
Gaussian distribution maximizes the differential entropy when the
second moment is constrained, in
\cref{lem:squared-l2-rate-dist} we show that
\begin{equation}\label{eq:sketch2}
  h(\targetPoint)-\frac{d}{2}\log\left(\frac{2\pi e D}{d}\right)
  \leq I(\targetPoint;\attackPoint).
\end{equation}
Combining \eqref{eq:sketch1}-\eqref{eq:sketch2} yields the claimed
lower bound. \qed

\paragraph{Comparison with prior bounds.}
\citet[Corollary~4]{balleReconstructingTrainingData2022} bound the
probability of reconstruction within a radius $r$ through
the small-ball probability
$\sup_v\prob(\|\targetPoint-v\|_2\leq r)$. %
In contrast, our bound is on the
mean squared error and the distribution of $x_t$ enters through its
differential entropy.
We note that the examples of a uniform and Gaussian distribution in
Propositions 6-7 of \cite{balleReconstructingTrainingData2022} also
show that $\rho\ll d$ guarantees the impossibility of reconstructing
training data. However, Fig.\
\ref{fig:comparison-reconstruction-bounds} demonstrates that the bound of
Theorem \ref{theorem:rho-entropy-reconstruction} is tighter for the Gaussian case and also for the uniform case unless $\rho/d$ approaches $0$. The improvement is particularly noticeable for moderate values of $\rho/d$, i.e., when a strong privacy
guarantee is required.

\begin{figure}[t]
  \centering
  \includegraphics[width=.9\linewidth]{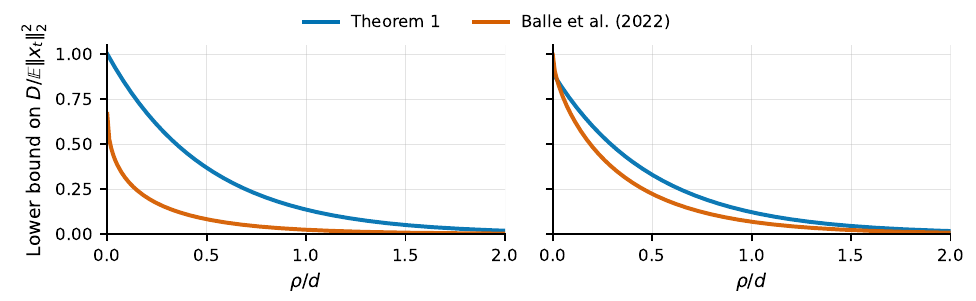}\par
  \vspace{-1.5em}
  \noindent\makebox[\linewidth][l]{%
    \hspace*{0.11\linewidth}%
    \subcaptionbox{Gaussian}[0.386\linewidth]{\rule{0pt}{0pt}}%
    \hspace{0.046\linewidth}%
    \subcaptionbox{Uniform in ball radius $\sqrt
    d$}[0.386\linewidth]{\rule{0pt}{0pt}}%
  }
  \vspace{-0.5em}
  \caption{Lower bounds on the normalized mean squared error $D/\ex
    \| \targetPoint\|_2^2$, with
    $D=\ex\|\targetPoint-\attackPoint\|_2^2$. We take $d=20$ and let
    the distribution of $x_t$ be standard Gaussian (left) or uniform
    on the ball of radius
    $\sqrt d$ (right), following the setting of Propositions 6-7 of
    \cite{balleReconstructingTrainingData2022}. The guarantees of
    \cite{balleReconstructingTrainingData2022} are translated into
    bounds on the mean squared error using
    the Gaussian Chernoff and uniform-ball volume bounds, see
    Appendix \ref{app:comparison-bounds} for details.
  }
  \label{fig:comparison-reconstruction-bounds}
\end{figure}

\subsection{Algorithmic upper bound}\label{sec:algo}

We next design an attack on the linear model in \eqref{eq:datamodel} privatized
via output perturbation (see \eqref{eq:clippedminimization}), and give provable guarantees on its reconstruction error. %

\begin{assumption}%
  \label{ass:full-dim-attack-assumption}
  Each feature datapoint $x \in \reals^d$ in $X$ is independently and
  identically drawn from a $O(1)$-subGaussian
  distribution  with
  bounded norm $\| x \|_2 \leq \featureRadiusBound$ for
  $\featureRadiusBound=\Theta(\sqrt{\dimInputs})$. Furthermore,
 we assume $\| x_t \|_2 = \Theta(\sqrt d)$ and $\labelNoise_i\sim
  \N(0,\labelStdDev^2)$.
\end{assumption}

In words, we require the data to have well-behaved tails, as commonly
done in related work
\citep{brown2024private, iurada2025law}.
We also note that this assumption can be relaxed to requiring $\n
\restOfDS^\top \restOfDS\n_{\op} = O(n+d)$.
The normalization $R = \Theta(\sqrt d)$ is chosen for convenience,
and it corresponds to taking the entries of $x$ of
constant order. The Gaussian assumption on $\nu_i$ is also chosen for
simplicity, as it can be relaxed to requiring sufficient
anti-concentration of the noise.

Note that the derivative
of the Huber loss
is $\psi_{\clipConstant}(a)
  :=\huberLoss'(a)
  =\sign(a)\min\{|a|,\clipConstant\}$.
Given the
non-target data,
define the leave-one-out gradient
\begin{equation}
  \label{eq:huber-loo-gradient-map}
  \huberLooGrad(\theta)
  :=\sum_{i \in [n], i\neq t}
  x_i\psi_{\clipConstant}(x_i^\top\theta-y_i)
  +n\lambda\theta.
\end{equation}
For $\lambda > 0$, the regularized objective in the RHS of
\eqref{eq:clippedminimization} is strongly
convex, and it has a unique minimizer $\optWeights$.
The first-order optimality condition thus gives %
\begin{equation}
  \label{eq:huber-loo-grad-at-full-opt-weights}
  -\huberLooGrad(\optWeights)
  =\targetPoint\,
  \psi_{\clipConstant}
  (\targetPoint^\top\optWeights-\targetLabel).
\end{equation}
Then, given access to $\optWeights$, through the
knowledge of $\restOfDS$,
$\restOfLabels$, $\clipConstant$ and $\lambda$, one can
compute $\huberLooGrad(\optWeights)$. This quantity, assuming
$\psi_{\clipConstant} (\targetPoint^\top\optWeights-\targetLabel)
\neq 0$, is an
exact estimator of the direction of $\targetPoint$, due to
\eqref{eq:huber-loo-grad-at-full-opt-weights}.
As the attacker has access to the perturbed weights
$\noisyWeights$ (and not to $\optWeights$), it is natural to pick
$\huberLooGrad(\noisyWeights)$ to estimate the direction of $\targetPoint$. %
The performance of this
reconstruction attack is characterized below.

\begin{theorem}
  [store=thm:direct-huber-reconstruction]
  \label{thm:direct-huber-reconstruction}
  Let \cref{ass:full-dim-attack-assumption} hold, and %
  let  $\lambda>0$,
  $\clipConstant>0$ and $\rho>0$ be the parameters of the algorithm
  chosen by the learner.
  The attacker computes
  \begin{equation}
    \label{eq:attack-definition}
    \attackPoint:=\featureRadiusBound\,\frac{\huberLooGrad(\noisyWeights)}{\|\huberLooGrad(\noisyWeights)\|_2}.
  \end{equation}
  Let $\eta\in(0,1)$ be
  a failure probability and $C , c > 0$ be absolute constants. %
  Then, with
  probability at least
  $1-\eta-2e^{-c n} - 2 e^{-c d}$,
  \begin{equation}
    \label{eq:huber-reconstruction-bound}
    \frac{\reconstructionError}{\| x_t \|_2^2}
    \leq
    C
    \left[
      \max\left\{1,
        \frac{\clipConstant}{\labelStdDev\eta}
        \left(1+\frac{d}{n\lambda}\right)
      \right\}
      \left(1+\frac1\lambda\right)
      \left(1+\frac dn\right)
    \right]^2
    \frac{d}{\rho}.
  \end{equation}
\end{theorem}

Theorem \ref{thm:direct-huber-reconstruction} states that, for large enough privacy parameter $\rho$, the estimate
$\huberLooGrad(\tilde \theta)$ is approximately aligned to the target
$x_t$. We note that the remaining sign and scale ambiguities may be easily resolvable in practice: the attacker has only two signs to consider (which would reconstruct the original image or its negative), and prior knowledge of the typical sample norm is a natural choice of scale.
Formally, we remark that it is also possible to extend the lower bound in Theorem \ref{theorem:rho-entropy-reconstruction} to the metric considered here, without losing the dependence on the dimension.

\begin{figure}[t]
  \includegraphics[width=\linewidth]{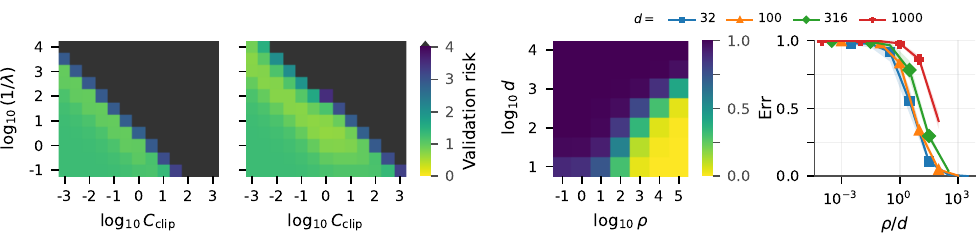}\par
  \vspace{-2em}
  \captionsetup[subfigure]{justification=centering,singlelinecheck=false}
  \noindent\makebox[\linewidth][l]{%
    \subcaptionbox{\label{fig:combined-synth-hyperparameter-search}
    Hyperparameter search at $\rho=10$ and $\rho=10^4$}%
    [0.48\linewidth]{\rule{0pt}{0pt}}\hfill
    \subcaptionbox{\label{fig:combined-synth-reconstruction}
    Reconstruction error relative to $\rho$ and $d$}%
    [0.48\linewidth]{\rule{0pt}{0pt}}%
  } 

\caption{Hyper-parameter selection and reconstruction, for synthetic data sampled uniformly on the sphere.
The first two panels show validation mean squared error (MSE) over a hyper-parameter grid at
$n=d=1000$, $\zeta = 0.5$, for $\rho=10$ and $\rho=10^4$. The third and fourth panel show normalized reconstruction error after hyper-parameter selection within the range  $\lambda\geq 0.1$ and
$\clipConstant\leq 0.1\labelStdDev$.
}
\label{fig:combined-synth-viz}
\end{figure}

Consider now the case where $\eta$ is a small constant (e.g., $\eta=0.01$), the
label noise is of constant order $\zeta = \Theta(1)$ and the number
of training samples satisfies $n = \Omega(d)$, often
needed in linear regression to achieve non-trivial test loss \citep{wainwright2019high, mourtada2022}. Then, the smallest $\rho$ allowing successful
reconstruction depends on the
hyper-parameters $\clipConstant, \lambda$ used by output
perturbation.
\cref{fig:combined-synth-hyperparameter-search}
shows that small clipping constants $\clipConstant = O(1)$ and a sufficiently
large regularization $\lambda = \Theta(1)$ empirically yield good
performance and, under this scaling, \eqref{eq:huber-reconstruction-bound} reads
\begin{equation}\label{eq:attackthmcompact}
\frac{\reconstructionError}{\| x_t \|_2^2} \lesssim \frac{d}{\rho}.
\end{equation}
This matches the scaling of the impossibility result of Theorem
\ref{theorem:rho-entropy-reconstruction}, and the sharp transition in data reconstruction at $\rho\asymp d$ is clearly displayed in \cref{fig:combined-synth-reconstruction}. 
We also remark that the findings in \cref{fig:combined-synth-viz}(a) on $\clipConstant$ are in analogy with the evidence that
small gradient clipping constants yield good downstream performance in private gradient methods \citep{li2022large, de2022, bombari2026highdimensionalprivatelinearregression}. %
The proof of Theorem \ref{thm:direct-huber-reconstruction} is
in \cref{subsec:huber-reconstruction-attack}, with a sketch
given below.

\paragraph{Proof sketch.}
The idea is to separate the signal carried by the target point from
the perturbation induced by the privacy noise.
Writing $a=\psi_{\clipConstant}(\targetPoint^\top\optWeights-\targetLabel)$,
the optimality condition \eqref{eq:huber-loo-grad-at-full-opt-weights} gives
\[
-\huberLooGrad(\noisyWeights)=a\targetPoint+w,
\qquad
w:=\huberLooGrad(\optWeights)-\huberLooGrad(\noisyWeights).
\]

First, to lower-bound $|a|$, we show that the target residual at
$\optWeights$ is close to that at the leave-one-out estimator $\restWeights$
defined in \eqref{eq:huber-loo-optimal-weights}.
Conditional on $(\targetPoint,\restOfPairsDS)$, the latter residual is
Gaussian with variance $\labelStdDev^2$.
Gaussian anti-concentration, combined with
\cref{lem:huber-loo-grad-difference-bounded} , then yields
\[
|a|\gtrsim
\min\left\{\clipConstant,
\frac{\labelStdDev\eta}{1+\featureRadiusBound^2/(n\lambda)}\right\}
\]
with probability at least $1-\eta$; see \eqref{eq:huber-signal-threshold}.
Next, the same lemma controls the amplification of the privacy noise
through $\huberLooGrad$.
Combining concentration of the Gram matrix $\restOfDS^\top\restOfDS$
and of the norm of the  Gaussian noise gives, with high probability,
\[
\|w\|_2\lesssim
\dpStdDev\sqrt d\,(n\lambda+n+d),
\]
as established in \eqref{eq:huber-transformed-privacy-noise-bound}.
Finally, \cref{lem:reconstruction-error-snr} gives
$\reconstructionError\leq4\|w\|_2^2/a^2$.
Substituting $\dpStdDev$ from \cref{prop:output-pert-is-private},
using $\featureRadiusBound^2=O(d)$, and taking a union bound
yields the final guarantee. \qed

\subsection{Extension to low-dimensional data}\label{sec:low_dim}

The bounds above show that $\rho \asymp d$ is a sharp threshold for
reconstructing data in dimension $d$ privatized via a $\rho$-zCDP
mechanism:  Theorem \ref{theorem:rho-entropy-reconstruction} implies
that reconstruction is information-theoretically impossible for $\rho
\ll d$, and Theorem
\ref{thm:direct-huber-reconstruction} exhibits a successful attack for
$\rho \gg d$. However, the impossibility result, and
therefore the tightness of this threshold, relies on the entropy
of $x_t$ to scale with the number of dimensions, i.e.,  $h(x_t)
\asymp d$, which may not be the case for data in practical
applications. To model that, we now focus on data lying in a lower
dimensional subspace of
$\mathbb R^d$. %

More precisely, let $P$ be a fixed orthogonal projection of rank $s <
d$, and let
$U\in\reals^{d\times s}$ have orthonormal columns spanning
$\operatorname{Im}(P)$, so that $P=UU^\top$. Define
\[
h_s(\targetPoint):=h(U^\top\targetPoint),
\]
where $h$ is taken with respect to the Lebesgue measure on $\reals^s$.
Note that this definition does not depend on the specific choice of
orthonormal basis, and its purpose is to avoid computing the
differential entropy of $P \targetPoint$ in the original ambient
space $\reals^d$, where its distribution would be singular. 

The result below (proved in \cref{sec:lower-bounds-sparse-priors})
extends the lower bound of Theorem \ref{theorem:rho-entropy-reconstruction}.

\begin{theorem}
[store=low-dim-entropy-thm]%
\label{theorem:low-dim-entropy-bound}
In the same setting as \cref{theorem:rho-entropy-reconstruction},
for any rank-$s$ projection $P$ satisfying the conditions above, any
reconstruction of $\targetPoint$ from a $\rho$-zCDP mechanism
obeys:
\begin{align}
\label{eq:sparse-reconstruction-lower-bound-relative}
\frac{\ex\|\targetPoint-\attackPoint\|_2^2}
{\ex\|P \targetPoint\|_2^2}
&\geq
\frac{1 }{2\pi e}
\exp\left(\frac{2}{s}(h_s(x_s')-\rho)\right),
\end{align}
with %
$x_s' := \frac{\sqrt
s}{\sqrt{\ex \|P \targetPoint\|_2^2}} P\targetPoint$.
\end{theorem}

We note that the lower bound in
\eqref{eq:sparse-reconstruction-lower-bound-relative} is analogous to
that in \eqref{eq:lower-bound-scale-invariant} upon replacing \emph{(i)} $\ex \n
\targetPoint\n_2^2$ with $\ex\|P \targetPoint\|_2^2$ in the
denominator of the LHS, and \emph{(ii)} $h(x_t')$ with $h_s(x_s')$ on
the RHS. Thus, if the point to be reconstructed $x_t$ lies close to
the span of $\operatorname{Im}(P)$ (i.e., $\ex \n
\targetPoint\n_2^2\approx \ex\|P \targetPoint\|_2^2$), then any
attack fails whenever $\rho$ is small compared to the differential
entropy $h_s(x_s')$ computed on that span. Furthermore, as this span
has dimension $s$, we would typically have $h_s(x_s')=\Theta(s)$,
implying the impossibility of data reconstruction whenever $\rho\ll s$.

To complement this lower bound, the result below (also proved in
\cref{sec:lower-bounds-sparse-priors}) characterizes the performance of a variant of the
attack in \cref{thm:direct-huber-reconstruction} that seeks to
reconstruct an $s$-dimensional projection of the target. %

\begin{theorem}
[store=thm:low-dim-huber-reconstruction ]
\label{thm:low-dim-huber-reconstruction}
Consider the same setting as
\cref{thm:direct-huber-reconstruction}, but with
$\featureRadiusBound=\Theta(\sqrt s)$ and
$\|\targetPoint\|_2=\Theta(\sqrt s)$ replacing the corresponding
$\Theta(\sqrt d)$ assumptions.
Let $P \in \mathbb R^{d\times d}$ be a rank-$s$ projector,
known to the attacker, such that
$x_i \in \mathrm{Im}(P)$ for all $i \in [n]$.
Let %
\begin{equation}\label{eq:projattack}
\hat x_P:=\featureRadiusBound
\frac{P\huberLooGrad(\noisyWeights)}
{\|P\huberLooGrad(\noisyWeights)\|_2}.
\end{equation}
Then, if $n = \Omega(s), \lambda = \Omega(1), \clipConstant =
O(\labelStdDev)$, we have
\begin{equation}
\label{eq:low-dim-huber-reconstruction-bound}
\frac{\min_{\tau\in[-1,1]}
\|\targetPoint-\tau\hat x_P\|_2^2}{\|\targetPoint\|_2^2}
\lesssim \frac{s}{\rho},
\end{equation}
with probability at least $0.99$.

\end{theorem}

While the result of Theorem \ref{theorem:low-dim-entropy-bound} is scale-invariant, Theorem \ref{thm:low-dim-huber-reconstruction} considers $R$ and $\|x_t\|_2$ of order $\sqrt{s}$.\footnote{If one wishes to keep those quantities of order $\sqrt{d}$, we note that the same result would hold upon taking $\lambda$ of order $d/s$ and assuming that, for any feature $x\in \mathbb R^d$, $\sqrt{s/d}x$ is $O(1)$-subGaussian.} Modulo this scale difference, the upper bound in
\eqref{eq:low-dim-huber-reconstruction-bound} is analogous to that in
\eqref{eq:attackthmcompact} upon replacing \emph{(i)} the estimator
$\attackPoint$ with its projected version $\hat x_P$, and
\emph{(ii)} the dimension $d$ of the ambient space with the dimension
$s$ of the subspace containing the data. In words, Theorem
\ref{thm:low-dim-huber-reconstruction} shows that an attack having
knowledge of the subspace of dimension $s$ in which the data lies, is
successful whenever $\rho\gg s$ and, combined with Theorem
\ref{theorem:low-dim-entropy-bound}, this establishes $\rho\asymp s$
as a sharp threshold for data reconstruction.

\section{Experiments}
\label{sec:experiments}

We study the attacks of Sections \ref{sec:algo}-\ref{sec:low_dim} for
data sampled uniformly on the $d$-dimensional sphere, synthetic data concentrated near an $s$-dimensional subspace, and binary regression on data taken from ImageNet (elephants vs.\ pandas) and CIFAR-10 (frog vs.\ trucks).
Details on datasets construction, learning algorithm, and attacks are given in \cref{app:experimental-details}. %
We report the squared reconstruction error (see e.g. the LHS of \eqref{eq:low-dim-huber-reconstruction-bound}), %
with the sign ambiguity resolved.

\paragraph{Synthetic $d$-dimensional data.} %
In \cref{fig:combined-synth-hyperparameter-search}, we investigate how different choices of $\clipConstant$ and $\lambda$ affect the performance of the output perturbation algorithm in \eqref{eq:output-perturbation}. The heatmaps show that, for both values of $\rho$, the best performance is achieved for \emph{(i)} sufficiently small values of $\clipConstant$, and \emph{(ii)} an appropriate ratio $\clipConstant / \lambda$. In fact, the lowest validation risk occurs on a top-left to bottom-right diagonal of the heatmap. Then, for any output perturbation algorithm, we take the values of $\lambda \ge 0.1$ and $\clipConstant \leq 0.1 \zeta$ that minimize the validation risk. %
In the left panel of \cref{fig:combined-synth-reconstruction}, we plot the reconstruction error as a function of $\rho$ and $d$, and identify a clear diagonal boundary $\rho / \dimInputs \asymp 1$ between successful and unsuccessful attacks, validating our result in \eqref{eq:attackthmcompact}. In the right panel of \cref{fig:combined-synth-reconstruction}, we plot the same reconstruction error as a function of $\rho / d$, and show that curves corresponding to different values of $d$ (in different colors) collapse onto each other under this scaling.

\begin{figure}[!t]
\centering
\begin{minipage}[t]{0.48\linewidth}
\vspace{0pt}
\centering
\captionsetup[subfigure]{justification=centering,singlelinecheck=false}
\includegraphics[width=\linewidth]{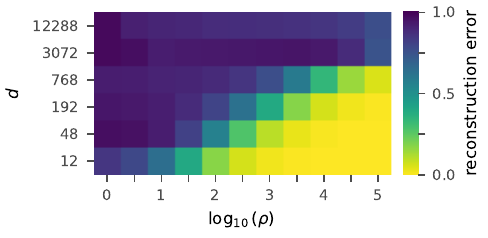}\par
  \vspace{-1em}
\subcaptionbox{\label{fig:imagenet-heatmap}ImageNet reconstruction
heatmap}[\linewidth]{\rule{0pt}{0pt}}\hfill

\includegraphics[width=\linewidth]{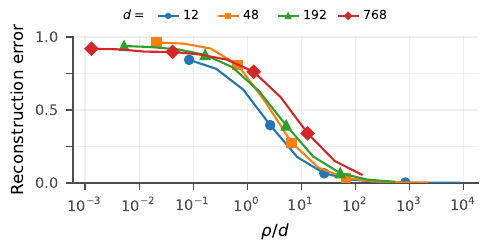}\par
  \vspace{-1em}
\subcaptionbox{\label{fig:imagenet-rho-d}ImageNet,
$\rho/d$}[\linewidth]{\vspace{-1em}\rule{0pt}{0pt}}\hfill

\caption{Reconstruction of ImageNet data (elephants vs.\ pandas).
Hyperparameter search for mean validation MSE within $\lambda\geq0.1$ and
$\clipConstant\leq0.01$. The panels show squared relative error against
$(\rho,\dimInputs)$ and against $\rho/\dimInputs$. 
\label{fig:imagenet}
}
\end{minipage}\hfill
\begin{minipage}[t]{0.48\linewidth}
\vspace{0pt}
\centering
\includegraphics[width=\linewidth]{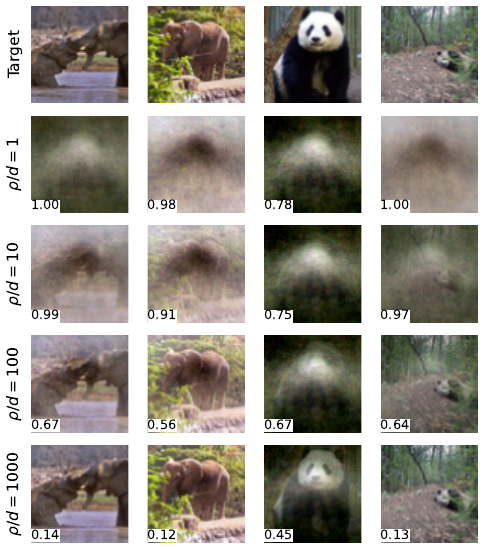}

\caption{Reconstructions of four target images from ImageNet (elephants vs.\ pandas) as the privacy budget grows,
with $\lambda=1$ and $\clipConstant=0.125$. The number inside
each reconstruction is its reconstruction error. Here,
$\dimInputs=12{,}288$ ($64\times64$ RGB images). Additional plots appear in
\cref{fig:visual-reconstruction-seeds,fig:cifar10-visual-reconstruction-seeds}.
\label{fig:visual-reconstruction} 
}
\end{minipage}
\end{figure}

\paragraph{Natural images.}
In \cref{fig:imagenet},  we consider ImageNet data, and perform the same experiments as those described in \cref{fig:combined-synth-viz} for synthetic data. Here, we tune the value of $d$ by down-scaling images to a lower resolution, and we optimize the output perturbation algorithm across the hyper-parameters ranges $\lambda \geq 0.1$ and $\clipConstant \leq 0.01$. 
In \cref{fig:imagenet-heatmap}, we plot the reconstruction error as a function of $\rho$ and $d$, and identify the same clear diagonal boundary as in \cref{fig:combined-synth-viz} for the attack success across different image resolutions. In \cref{fig:imagenet-rho-d}, we then plot the same reconstruction error as a function $\rho/\dimInputs$, noting that curves corresponding to different values of $d$ collapse onto each other. This provides evidence that our proposed threshold $\rho\asymp d$ for data reconstruction persists in natural images as well.  Furthermore, in 
\cref{fig:visual-reconstruction} we give a visual presentation of some target images (first row) and the corresponding reconstructions for increasing values of $\rho / \dimInputs$.
Additional visual reconstructions appear in
\cref{fig:visual-reconstruction-seeds,fig:cifar10-visual-reconstruction-seeds} (Appendix \ref{app:visual-reconstruction-seeds}).

\paragraph{Low-dimensional data.} 

In \cref{fig:pca-reconstruction}, we numerically investigate the conclusions of \cref{thm:low-dim-huber-reconstruction,theorem:low-dim-entropy-bound}. 
More precisely, in \Cref{fig:pca-reconstruction-synthetic}, we fix the ambient dimension at $\dimInputs=1000$, and generate synthetic data concentrated near an $s$-dimensional subspace. The construction is detailed in Appendix \ref{app:datasets}. We test the reconstruction attack in \eqref{eq:projattack}, estimating the rank-$s$ PCA projector $P$ from $\restOfDS$, and report its performance as a function of $\rho / s$ (left panel) and $\rho$ (right panel). The reconstruction error curves with respect to $\rho / s$ overlap for different values of $s$, showing a transition at the level of $\rho \asymp s$.
In contrast, the curves are clearly separated when plotted against
$\rho$. This validates our bound in \eqref{eq:low-dim-huber-reconstruction-bound}, and it shows that the same privacy budget $\rho$ can result in different guarantees in terms of data reconstruction, depending on the effective dimension of the data.

Next, in \cref{fig:pca-reconstruction-cifar10}, we consider a similar class of experiments, training our DP linear regression model on CIFAR-10 data. We perform the reconstruction attack in \eqref{eq:projattack}, defining $P$ as the projector over the $s$ principal components of the data distribution, estimated from $\restOfDS$. %
We again plot the reconstruction error both in terms of $\rho / s$ (left panel) and $\rho$ (right panel). As for the synthetic data, the curves are separated when plotted against $\rho$, and they collapse onto each other when plotted against $\rho/s$. However, differently from the synthetic data, when $s$ is too small (e.g., $s=10$), the error exhibits a plateau even at large $\rho/s$,
capturing the fact that low-rank subspaces do not capture all of the signal of the target.
Similar results for ImageNet are reported in
\cref{fig:pca-reconstruction-imagenet} (Appendix \ref{app:pca-reconstruction-rest}).

\begin{figure}[t]
\centering
\includegraphics[width=0.48\linewidth]{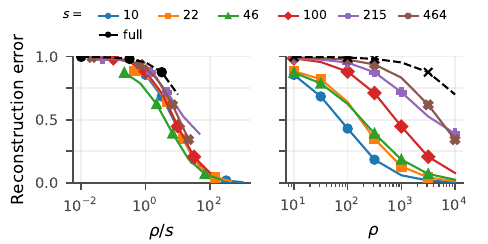}\hfill
\includegraphics[width=0.48\linewidth]{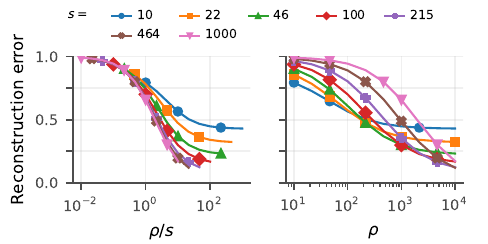}\par
\vspace{-1em}
\captionsetup[subfigure]{justification=centering,singlelinecheck=false}
\noindent\makebox[\linewidth][l]{%
\subcaptionbox{\label{fig:pca-reconstruction-synthetic}Synthetic
rank-$s$ data}%
[0.48\linewidth]{\rule{0pt}{0pt}}\hfill
\subcaptionbox{\label{fig:pca-reconstruction-cifar10} CIFAR-10 (frog vs.\ truck)}%
[0.48\linewidth]{\rule{0pt}{0pt}}%
}
\caption{Reconstruction error for projected attacks on (a) synthetic
rank-$s$ data with $\dimInputs=1000$, and (b) CIFAR-10
(frog vs.\ truck) data with $\dimInputs=3{,}072$ ($32\times32$ RGB images). Each pair of plots uses $\rho/s$ (left) and $\rho$ (right) on the x-axis, and different colors for different values of the rank  $s$. %
\label{fig:pca-reconstruction}}
\end{figure}

\section{Conclusions}

In this work, we study how differential privacy quantitatively determines the feasibility of successful reconstruction attacks.
First, in \cref{theorem:rho-entropy-reconstruction}, we prove that any $\rho$-zCDP mechanism prevents data reconstruction as long as the privacy budget is sufficiently smaller than the entropy of the target, i.e., $\rho \ll h(x_t)$. When $h(x_t) = \Theta(d)$ (which holds, e.g., for distributions with independent entries), this provides the sufficient condition $\rho \ll d$ to protect from reconstruction attacks.
Then, in Theorem \ref{thm:direct-huber-reconstruction}, we show that this condition is also necessary: there exists an attack on a linear model trained with output perturbation that achieves arbitrary accuracy as $\rho \gg d$.
Taken together, these results identify $\rho \asymp \dimInputs$ as a sharp transition in data
reconstruction. %
We next show that this threshold shifts according to the \emph{effective} dimension of the data: when the data lies on an $s$-dimensional subspace known to the attacker, the transition moves to $\rho \asymp s$ (\cref{thm:low-dim-huber-reconstruction,theorem:low-dim-entropy-bound}). This highlights the role of the data prior in assessing the effective guarantees given by a privacy budget.

A first natural research direction for future work is to extend \cref{thm:direct-huber-reconstruction} to other algorithms (e.g., DP gradient descent or objective perturbation) and beyond the linear model assumption, as well as expand our results on low-dimensional data to non-linear subspaces.
A second exciting avenue is %
to consider less powerful attackers, without knowledge of all remaining training data.
Information-theoretic lower bounds in
that setting would complement the recent work by \citet{swanberg2026unifiedframeworkadversaryawaredifferential} on adversary-aware privacy
guarantees.
We suspect that matching those bounds with attacks would require a different learning
model: when multiple training examples are unknown, the optimality conditions of a linear model generally leave their individual values
under-determined \citep{runkel2025training}. 
Thus, the success of the attack will likely depend on the capacity of the learning model, in agreement with the results of \cite{iurada2025law} in the context of non-private learning.

\newpage

\section*{Acknowledgements}
 
This research was funded in whole or in part by the Austrian Science Fund (FWF) 10.55776/COE12. For the purpose of open access, the authors have applied a CC BY public copyright license to any Author Accepted Manuscript version arising from this submission.

Simone Bombari was supported by a Google PhD fellowship. The authors would like to thank Edwige Cyffers for helpful discussions.

This research was supported by the Scientific Service Units (SSU) of the Institute of Science and
Technology Austria through resources provided by Scientific Computing (SciComp).

\bibliography{references.bib}

\bibliographystyle{plainnat}

\appendix
\crefalias{section}{appendix}
\section{Additional notation}

Given a symmetric matrix $A$, we denote with $\tr(A)$ its trace, and by $\svmin{A}$ ($\svmax{A}$) its
smallest (largest) singular value.
Given a symmetric matrix $A$, we
denote by $\evmin{A}$ ($\evmax{A}$) its smallest (largest) eigenvalue.
For square, symmetric matrices $A,B$, we write $ A\preceq B$ if $B-A$ is positive semi-definite (p.s.d.), where $A$ is p.s.d.\ if $x^\top A x \geq 0$ for all $x\in \reals^n$.
We use the notation $[b]_+=\max\{b,0\}$ for $b \in \reals$.

\section{Sensitivity and privacy of output perturbation}
\label{subsec:output-pert-sens}

\begin{lemma}[Sensitivity of \eqref{eq:clippedminimization}]
  \label{lem:output-perturbation-sensitivity}
Fix $\lambda,\featureRadiusBound,\clipConstant>0$. Let
$\fullPairsDS,\fullPairsDS'$ be adjacent datasets of size $\numPoints$,
with every feature vector in either dataset having norm at most
$\featureRadiusBound$. Then, the minimizer in
\eqref{eq:clippedminimization} satisfies
  \begin{equation}
    \left\|\optWeights(\fullPairsDS)-\optWeights(\fullPairsDS')\right\|_2
    \leq
    \frac{2 R \clipConstant}
    {\lambda\numPoints}.
  \end{equation}
\end{lemma}

\begin{proof}
  Without loss of generality, let $\fullPairsDS$ and $\fullPairsDS'$
  differ in their first
  example, and write
  \[
    \theta:=\optWeights(\fullPairsDS),
    \qquad
    \theta':=\optWeights(\fullPairsDS'),
    \qquad
    r:=\theta-\theta'.
  \]
Define
\[
  g_i(\vartheta):=
  x_i\ell_{\clipConstant}'(x_i^\top\vartheta-y_i),\qquad  g'_1(\vartheta):=
  x'_1\ell_{\clipConstant}'((x'_1)^\top\vartheta-y'_1).
\]
The first-order optimality conditions for the two strongly convex
  objectives are
  \begin{align*}
    0
    &=\frac{1}{\numPoints}
    \left(g_1(\theta)+\sum_{i=2}^{\numPoints}g_i(\theta)\right)
    +\lambda\theta,\qquad
    0
    =\frac{1}{\numPoints}
    \left(g'_1(\theta')+\sum_{i=2}^{\numPoints}g_i(\theta')\right)
    +\lambda\theta'.
  \end{align*}
  Subtracting these equations and taking the inner product with $r$
  gives
  \begin{align}\label{eq:sensitivitypert1}
    \lambda\|r\|_2^2
    ={}&-\frac{1}{\numPoints}
    r^\top \left(\sum_{i=2}^{\numPoints}
      \bigl(g_i(\theta)-g_i(\theta')\bigr)
    \right)-\frac{1}{\numPoints}
    r^\top \left(g_1(\theta)-g'_1(\theta')\right).
  \end{align}
  The first term on the RHS is non-positive. In fact, $\sum_{i=2}^{\numPoints}
  \ell_{\clipConstant}(x_i^\top\vartheta-y_i)$
  is convex in $\vartheta$, so its gradient is monotone, i.e.,
  \[
    r^\top \left(\sum_{i=2}^{\numPoints}
      \bigl(g_i(\theta)-g_i(\theta')\bigr)
    \right)
    \geq 0.
  \]

  Since $|\ell_{\clipConstant}'(\cdot)|\leq\clipConstant$ and
  $\|x_i\|_2\leq\featureRadiusBound$, every per-example gradient
  satisfies
  \[
    \|g_i(\vartheta)\|_2
    \leq R \clipConstant.
  \]
  Then, plugging this result into \eqref{eq:sensitivitypert1},
  Cauchy-Schwarz and the triangle inequality yield
  \begin{align*}
    \lambda\|r\|_2^2
    &\leq\frac{\|r\|_2}{\numPoints}
    \left(\|g_1(\theta)\|_2+\|g'_1(\theta')\|_2\right)\\
    &\leq
    \frac{2 R \clipConstant}{\numPoints}\|r\|_2.
  \end{align*}
  If $r=0$, the claim is immediate. Otherwise, dividing by
  $\lambda\|r\|_2$ proves
  \[
    \| r\|_2 = \|\optWeights(\fullPairsDS)-\optWeights(\fullPairsDS')\|_2
    \leq
    \frac{2 R \clipConstant}
    {\lambda\numPoints},
  \]
  which gives the desired result.
\end{proof}

\getkeytheorem{prop:output-pert-is-private}

\begin{proof}
  By \cref{lem:output-perturbation-sensitivity}, for adjacent
  $\fullPairsDS,\fullPairsDS'$,
  \[
    \left\|\optWeights(\fullPairsDS)-\optWeights(\fullPairsDS')\right\|_2
    \leq
    \Delta_2
    :=\frac{2 R \clipConstant}
    {\lambda\numPoints}.
  \]
  For every $\alpha>1$, the Rényi divergence between Gaussians with
  means $\mu,\mu'$ and common covariance $\dpStdDev^2I_{\dimInputs}$
  is given by (see \citep[Proposition 7]{mironovRenyiDifferentialPrivacy2017})
  \[
    D_\alpha\!\left(
      \N(\mu,\dpStdDev^2I_{\dimInputs})
      \middle\|
      \N(\mu',\dpStdDev^2I_{\dimInputs})
    \right)
    =\frac{\alpha\|\mu-\mu'\|_2^2}{2\dpStdDev^2}.
  \]
  Taking $\mu=\optWeights(\fullPairsDS)$ and
  $\mu'=\optWeights(\fullPairsDS')$, the stated
  lower bound on $\dpStdDev^2$ gives
  \[
    D_\alpha\bigl(\noisyWeights(\fullPairsDS)\|\noisyWeights(\fullPairsDS')\bigr)
    \leq
    \frac{\alpha\Delta_2^2}{2\dpStdDev^2}
    \leq\rho\alpha.
  \]
  The claim follows from \cref{def:zcdp}.
\end{proof}

Note, of course, that larger variances than the value stated also satisfy $\rho$-zCDP given the same hyperparameters, but lead to greater-than-necessary utility loss; we therefore assume the learner chooses the smallest variance to satisfy their desired privacy bound.

\newcommand{\rate}{\mathsf{R}}
\newcommand{\sourceRV}{S}

\section{Proof of the information-theoretic lower bound}
\label{sec:lower-bounds-entropy-mi}

We start with a lemma using arguments from
\cite{bunConcentratedDifferentialPrivacy2016}. %

\begin{lemma}
  [Mutual information bound from $\rho$-zCDP]
  \label{lem:zcdp-mut-info}

  Suppose $\mech$ is a $\rho$-zCDP mechanism, applied to a dataset
  $\fullPairsDS = \restOfPairsDS \cup
  \{\targetPair\}$, where $\targetPair$ is
  a random variable and $\restOfPairsDS$ is fixed. Then, we have
  \begin{align}
    I(\targetPair; \mech(\restOfPairsDS \cup \{\targetPair\})) \leq \rho,
  \end{align}
  where $I(\cdot ; \cdot)$ denotes the mutual information in the
  probability space of $\targetPair$ and of the private mechanism.
\end{lemma}
\begin{proof}%
  As $\mech$ is $\rho$-zCDP, due to Definition
  \ref{def:zcdp}, we have that
  \[
    D_\alpha(\mech(\fullPairsDS) \| \mech(\fullPairsDS')) \leq \rho \cdot \alpha,
  \]
  for all $\alpha > 1$ and all neighboring datasets
  $\fullPairsDS,\fullPairsDS'$ where $\fullPairsDS = \restOfPairsDS \cup \{z\}$, $\fullPairsDS'
  = \restOfPairsDS\cup \{z'\}$, for arbitrary $z,z' \in \inputSpace$.
  By continuity of the Rényi divergence as $\alpha \to 1$ (on
  $\alpha \in [0,\infty]$ when $D_\alpha(P \| Q) < \infty$)
  \cite[Theorem 7]{ervenRenyiDivergenceKullbackLeibler2014}, the
  above condition gives
  \begin{align}\label{eq:privacyKL}
    \kl(\mech(\fullPairsDS) \| \mech(\fullPairsDS')) = D_1(\mech(\fullPairsDS) \|
    \mech(\fullPairsDS')) \leq \rho.
  \end{align}
  For shorthand, we introduce the notation $\noisyWeights =
  \mech(\restOfPairsDS \cup \{\targetPair\})$, $\theta$ for a value
  taken by $\noisyWeights$, and $z$ for a value taken by $\targetPair$. Then, we have
  \begin{align*}
    I(\targetPair;  \noisyWeights)
    &= \int
    p_{\targetPair}(z)
    p_{\noisyWeights| \targetPair}(\theta | z)
    \log \frac{
      p_{\noisyWeights| \targetPair}(\theta | z)
    }{
      p_{\noisyWeights}(\theta)
    }
    \dd z \,
    \dd \theta  \\
    &= \int
    p_{\targetPair}(z)
    \left(
      \int
      p_{\noisyWeights| \targetPair}(\theta | z)
      \log \frac{
        p_{\noisyWeights| \targetPair}(\theta | z)
      }{
        p_{\noisyWeights}(\theta)
      }
      \dd \theta
    \right)
    \dd z \\
    &= \int
    p_{\targetPair}(z)
    \kl( (\noisyWeights \mid \targetPair=z) \, \| \, \noisyWeights)
    \dd z \\
    &= \int
    p_{\targetPair}(z)
    \kl( (\noisyWeights \mid \targetPair=z) \, \| \, \ex_{\targetPair}
    [\noisyWeights \mid \targetPair] )
    \dd z \\
    &\leq \int
    p_{\targetPair}(z)
    p_{\targetPair} (z') \left(
      \int
      p_{\noisyWeights| \targetPair}(\theta | z)
      \log \frac{
        p_{\noisyWeights| \targetPair}(\theta | z)
      }{
        p_{\noisyWeights| \targetPair}(\theta | z')
      }
      \dd \theta
    \right)
    \dd z'
    \dd z
    \\
    &= \int
    p_{\targetPair}(z)
    p_{\targetPair} (z')
    \kl(\mech(\fullPairsDS) \| \mech(\fullPairsDS'))
    \dd z'
    \dd z
    \\
    &\leq \int
    p_{\targetPair}(z)
    p_{\targetPair} (z')
    \rho
    \dd z'
    \dd z,
    \\
    & = \rho,
  \end{align*}
  where the fifth line holds as the KL divergence is convex in its
  second argument \cite[Theorem
  12]{ervenRenyiDivergenceKullbackLeibler2014} and the seventh line
  follows from \eqref{eq:privacyKL}.
\end{proof}

We can now turn to a brief discussion on rate-distortion
and distortion-rate functions, before we develop a theorem
combining this with the mutual information bound. For more details, we refer the interested reader to the classical textbook \citep{coverElementsInformationTheory2001}. Let $\sourceRV$ be a random variable in $\reals^k$, for an arbitrary
dimension $k\geq1$, and consider a (possibly random) function
$\channel: \reals^k \to \reals^k$ which we will refer to as
\emph{channel}. We define its rate and distortion by
\[
  \rateChan(\channel)
  := I(\sourceRV;\channel(\sourceRV)),
  \qquad
  \distChan(\channel)
  := \ex\bigl[ \| \sourceRV-\channel(\sourceRV) \|_2^2\bigr].
\]
Then, we define the \emph{rate-distortion} and \emph{distortion-rate}
functions as
\begin{align}\label{eq:ratedist}
    \rateDist(D)
  := \inf_{\channel: \distChan(\channel)\leq D}
  \rateChan(\channel),
  \qquad
  \distRate(\rate)
  := \inf_{\channel: \rateChan(\channel)\leq \rate}
  \distChan(\channel).
\end{align}
We will use this notation in the remaining part of the section.

\begin{lemma}
  [Rate/distortion bounds for the squared $\ell_2$ distortion]
  \label{lem:squared-l2-rate-dist}
  Let $D>0$ and $\rate \geq 0$. Then, the rate-distortion and distortion-rate
  functions defined in \eqref{eq:ratedist} satisfy 
  \begin{align}
    \rateDist(D)
    &\geq \left[h(\sourceRV) -
    \frac{k}{2} \log \left( 2 \pi e \frac{D}{k}\right)\right]_+, \\
    \distRate(\rate)
    &\geq \frac{k}{2\pi e} \exp \left(\frac{2}{k}
    (h(\sourceRV) - \rate) \right).
  \end{align}
\end{lemma}
\begin{proof}
   We have
  \begin{align*}
    \rateDist(D) & =
    \inf_{\channel: \ex \n \channel(\sourceRV) - \sourceRV\n_2^2 \leq D}
    I(\sourceRV; \channel(\sourceRV)) \\
    & =
    h(\sourceRV) -
    \sup_{\channel: \ex \n \channel(\sourceRV) - \sourceRV\n_2^2 \leq D}
    h(\sourceRV | \channel(\sourceRV))\\
    & =
    h(\sourceRV) -
    \sup_{\channel: \ex \n \channel(\sourceRV) - \sourceRV\n_2^2 \leq D}
    h(\sourceRV - \channel(\sourceRV) | \channel(\sourceRV))\\
    & \geq
    h(\sourceRV) -
    \sup_{\channel: \ex \n \channel(\sourceRV) - \sourceRV\n_2^2 \leq D}
    h(\sourceRV - \channel(\sourceRV) )\\
    & = h(\sourceRV) -
    \sup_{\Delta: \ex \n \Delta\n_2^2 \leq D}
    h(\Delta )\\
    & = h(\sourceRV) -
    \frac{k}{2} \log \left( 2 \pi e \frac{D}{k}\right).
  \end{align*}
Here, the second line uses $I(\sourceRV; \channel(\sourceRV)) = h(\sourceRV) - h(\sourceRV | \channel(\sourceRV))$; the third line uses that the conditional differential entropy is translation-invariant; the fourth line uses that conditioning decreases entropy; in the fifth line we take $\Delta = \sourceRV - \channel(\sourceRV)$ and use that the space of all
      random variables $\Delta$ with $\n \Delta\n_2^2 \leq D$ is
    equivalent to $\{\sourceRV-\channel(\sourceRV): \ex \n \channel(\sourceRV) - \sourceRV\n_2^2 \leq D  \}$; the sixth line uses that, under this distortion condition, the max-entropy
    distribution is Gaussian \cite[Theorem 9.6.5]{coverElementsInformationTheory2001}.

  Combining this with $\rateDist(D)\geq 0$ gives the first bound.
  For any channel with $\rateChan(\channel)\leq \rate$, we have
  \[
    h(\sourceRV)-\frac{k}{2}\log\left(2\pi e\frac{\distChan(\channel)}{k}\right)
    \leq \rateDist(\distChan(\channel))
    \leq \rateChan(\channel)\leq \rate.
  \]
  Rearranging and taking the infimum over such channels gives
  \[
    \distRate(\rate)\geq \frac{k}{2\pi e}
    \exp\left(\frac{2}{k}(h(\sourceRV)-\rate)\right).
  \]
\end{proof}

\getkeytheorem{rho-entropy-reconstruction}

Note that we assume that the differential entropy is well-defined, \ie that \(x_t\) has a density with respect to the Lebesgue measure on \(\mathbb R^d\), has finite differential entropy, and \(0<\mathbb E\|x_t\|_2^2<\infty\).
\begin{proof}
  Fix $\restOfPairsDS$ and write $\targetPair=(\targetPoint,\targetLabel)$.
  Consider the shorthand $\theta(z) = \mech(\restOfPairsDS \cup \{z\})$.
  The target pair $\targetPair$, the
  resulting learned weights $\theta(\targetPair)$, and the reconstruction
  $\attackPoint=\attack(\restOfPairsDS,\theta(\targetPair))$ form a Markov chain
  \[
    \targetPair \to \theta(\targetPair) \to \attackPoint.
  \]
  Since $\targetPoint$ is the feature component of $\targetPair$, we have
  \[
    I(\targetPoint; \attackPoint)
    \leq I(\targetPair; \attackPoint)
    \leq I(\targetPair; \theta(\targetPair)) \leq \rho,
  \]
  where the first two steps follow from the data-processing inequality,
  and the last step holds due to \Cref{lem:zcdp-mut-info}, since
  $\mech$ is $\rho$-zCDP. All information quantities here are evaluated
  conditional on the fixed $\restOfPairsDS$.

  Now, define the induced channel $\channel(\targetPoint) = \attackPoint$, averaging over the conditional
  distribution of $\targetLabel$ and the randomness of the mechanism and attack.
  We apply the rate-distortion definitions to the feature variable
  $\targetPoint$, taking $\sourceRV=\targetPoint$ and $k=\dimInputs$ in
  \cref{lem:squared-l2-rate-dist}.
  Since $\rateChan(\channel) = I(\targetPoint; \attackPoint) \leq \rho$, using the definition of distortion-rate function $\distRate$ in \eqref{eq:ratedist} and
  \cref{lem:squared-l2-rate-dist}, we have
  \[
    \ex \n \targetPoint - \attackPoint\n_2^2 = \distChan(\channel) \geq
    \distRate(\rho) \geq  \frac{d}{2 \pi e}
    \exp\left(\frac{2}{d}(h(\targetPoint) -\rho)\right).
  \]
  Finally, to remove the effect of the scale of $\targetPoint$, we apply the fact that $h(AX) = h(X)+ \log |\operatorname{det}(A)|$ \cite[Theorem 9.6.4]{coverElementsInformationTheory2001} with $\targetPoint' = \sqrt{\frac{d}{\ex \n \targetPoint\n_2^2}} \targetPoint$, and divide by $\ex\n\targetPoint\n_2^2$ to obtain \eqref{eq:lower-bound-scale-invariant}. 
\end{proof}

\section{Proof of the algorithmic upper bound}
\label{subsec:huber-reconstruction-attack}

\begin{lemma}
  \label{lem:reconstruction-error-snr}
  Let $\targetPoint \in \reals^d$, and let %
  $q \in \reals^d$ be such that $q = a \targetPoint + w \neq 0$,
  for some real number $a\neq 0$, and some $w\in\reals^d$.
  Defining $\attackPoint= \| x_t \|_2 \frac{q}{\|q\|_2}$, we have
  \[
    \reconstructionError
    \leq
    \frac{ 4 \|w\|_2^2}{a^2}.
  \]

\end{lemma}

\begin{proof}
  If $\targetPoint=0$, choose $\tau=0$. Otherwise, set
  $\targetRadius=\|\targetPoint\|_2>0$ and choose
  $\tau=\sign(a)\in[-1,1]$.
  Then
  \begin{align*}
    \frac{1}{\targetRadius}\|\targetPoint-\tau\attackPoint\|_2
    &=
    \left\|\frac{q}{\|q\|_2}
    -\frac{a\targetPoint}{\|a\targetPoint\|_2}\right\|_2 \\
    &\leq
    \left\|\frac{q}{\|q\|_2}
    -\frac{q}{\|a\targetPoint\|_2}\right\|_2
    +
    \left\|\frac{q-a\targetPoint}{\|a\targetPoint\|_2}\right\|_2 \\
    &=
    \frac{\bigl|\|a\targetPoint\|_2-\|q\|_2\bigr|}{\|a\targetPoint\|_2}
    +\frac{\|w\|_2}{\|a\targetPoint\|_2} \\
    &\leq
    \frac{2\|w\|_2}{|a|\targetRadius}.
  \end{align*}
  The last inequality follows from the reverse triangle inequality
  and $q=a\targetPoint+w$. Multiplying by $\targetRadius$, squaring,
  and minimizing over $\tau\in[-1,1]$ gives the desired result.
\end{proof}

\begin{lemma}
  [A property of $\huberLooGrad$]
  \label{lem:huber-loo-grad-difference-bounded}
  Let $\huberLooGrad$ be defined as in
  \eqref{eq:huber-loo-gradient-map}. Then, for any $\theta_1,
  \theta_2\in\reals^d$, we have
  \begin{equation}
    \label{eq:huber-gradient-difference-exact}
    \huberLooGrad(\theta_1)-\huberLooGrad(\theta_2)
    =
    M(\theta_1-\theta_2),
  \end{equation}
  for some positive semi definite matrix $M$ that satisfies
  \begin{equation}
    \label{eq:huber-difference-matrix-order}
    n \lambda I
    \preceq M
    \preceq \restOfDS^\top\restOfDS+n\lambda I.
  \end{equation}

\end{lemma}
\begin{proof}
  Define a diagonal matrix $D \in \reals^{(n - 1) \times (n - 1)}$
  with entries
  \[
    D_{ii}:=
    \begin{cases}
      \displaystyle
      \frac{
        \psi_{\clipConstant}(x_i^\top\theta_1-y_i)
        -\psi_{\clipConstant}(x_i^\top\theta_2-y_i)
      }{
        x_i^\top(\theta_1-\theta_2)
      },
      &\text{if } x_i^\top(\theta_1-\theta_2)\ne0,\\[3mm]
      0,
      &\text{otherwise},
    \end{cases}
  \]
  where we consider non-target points $\restOfDS = \{x_i\}_{i \in [n-1]}$.
Then,   \eqref{eq:huber-loo-gradient-map} yields
  \[
    \huberLooGrad(\theta_1)-\huberLooGrad(\theta_2) = \sum_{i \in [n-1]}
    x_i D_{ii} x_i^\top(\theta_1-\theta_2) + \lambda n (\theta_1 -
    \theta_2) = M (\theta_1 - \theta_2),
  \]
  with
  \begin{equation}
    M = \restOfDS^\top D\restOfDS+n\lambda I.
  \end{equation}
  Since the gradient of the clipped loss function
  $\psi_{\clipConstant}$ is non-decreasing and
  $1$-Lipschitz, we have
  $0\leq D_{ii}\leq1$, which implies that $\restOfDS^\top D\restOfDS$
  is p.s.d. and such that $\restOfDS^\top D\restOfDS \preceq
  \restOfDS^\top \restOfDS$.
\end{proof}

\getkeytheorem{thm:direct-huber-reconstruction}

\begin{proof}
  Consider the shorthand for the residuals
  \[
    r_t(\theta) := \targetPoint^\top\theta-\targetLabel.
  \]
  As noted in
  \eqref{eq:huber-loo-grad-at-full-opt-weights}, we have that
  \[
    -\huberLooGrad(\optWeights)
    =\targetPoint\,
    \psi_{\clipConstant}
    (\targetPoint^\top\optWeights-\targetLabel).
  \]
  On the perturbed noisy weights $\noisyWeights$, we have
  \begin{equation}
    \label{eq:huber-attack-signal-plus-noise}
    -\huberLooGrad(\noisyWeights) = \targetPoint\,\psi_{\clipConstant}
    (r_t(\optWeights)) -
    (\huberLooGrad(\noisyWeights) - \huberLooGrad(\optWeights)).
  \end{equation}
  \newcommand{\targetResidOpt}{r_t(\optWeights)}
  \newcommand{\targetResidRest}{r_t(\restWeights)}
  Firstly, we lower bound the absolute value of
  $\psi_{\clipConstant}(r_t(\optWeights))$.
  A useful tool is the optimal solution to the same learning problem on
  all datapoints except the target, \ie, just on $\restOfDS$, with a
  slightly altered regularization parameter $\frac{n}{n-1}\lambda$: %

  \begin{equation}
    \label{eq:huber-loo-optimal-weights}
    \restWeights:=\argmin_{\theta\in\reals^d}
    \left\{
      \frac{1}{n-1}\sum_{i\in[n], i\neq t}
      \huberLoss(x_i^\top\theta-y_i)
      +\frac{1}{2}\frac{n\lambda}{(n-1)}\|\theta\|_2^2
    \right\}.
  \end{equation}
  This objective is strongly convex, so its minimizer is unique,
  and its gradient is $\huberLooGrad/(n-1)$. Hence
  $\huberLooGrad(\restWeights)=0$.
  Due to \cref{lem:huber-loo-grad-difference-bounded}, we have
  \[
    \optWeights - \restWeights = M^{-1} (\huberLooGrad(\optWeights) -
    \huberLooGrad(\restWeights)) = M^{-1} \huberLooGrad(\optWeights),
  \]
  where the second step holds as
  $\huberLooGrad(\restWeights) = 0$. Note that $M \in \reals^{d
  \times d}$ is a p.s.d. matrix such that $\opnorm{M^{-1}} \leq (n
  \lambda)^{-1}$. Then, we have
  \begin{align}
    \targetResidOpt
    &= \targetPoint^\top \optWeights - \targetLabel \\
    &=
    \targetPoint^\top (\optWeights - \restWeights) +
    \underbrace{ \targetPoint^\top \restWeights - \targetLabel}_{:=
    \targetResidRest} \\
    &= \targetPoint^\top M^{-1}\huberLooGrad(\optWeights) +
    \targetResidRest
    \\
    &= -\targetPoint^\top M^{-1}\targetPoint
    \psi_{\clipConstant}(\targetResidOpt) + \targetResidRest.
    \label{eq:huber-two-residual-relation}
  \end{align}
  Since $\targetResidOpt$ and $\psi_{\clipConstant}(\targetResidOpt)$
  have the same sign, and
  $0 \leq \targetPoint^\top M^{-1}\targetPoint \leq
  \frac{\featureRadiusBound^2}{n\lambda}$,
  we have
  \[
    |\targetResidRest|
    = |\targetResidOpt|
    +\targetPoint^\top M^{-1}\targetPoint
    |\psi_{\clipConstant}(\targetResidOpt)|
    \leq
    \left(1+\frac{\featureRadiusBound^2}{n\lambda}\right)
    |\targetResidOpt|.
  \]
  Thus, for any $u\geq0$, we have
  $$\prob (|\targetResidOpt| \leq u ) \leq
  \prob\left(|\targetResidRest|\leq \left(1+
  \frac{\featureRadiusBound^2}{n\lambda}\right) u\right).$$
  The residual $\targetResidRest$ decomposes as
  \(
    \targetResidRest
    =
    \targetPoint^\top(\restWeights-\trueWeights)-\labelNoise_t.
  \)
  Conditional on $(\targetPoint,\restOfPairsDS)$, the first term is
  fixed and $\labelNoise_t\sim\N(0,\labelStdDev^2)$ remains independent.
  Its density is bounded by $1/(\labelStdDev\sqrt{2\pi})$, so
  \[
    \prob\left(
      |\targetResidRest|\leq\labelStdDev\eta\sqrt{\pi/2}
    \right)\leq\eta.
  \]

  Together, this gives
  \begin{equation}
    \label{eq:huber-signal-threshold}
    \prob(|\targetResidOpt| \leq u ) \leq \eta, \qquad
    u
    :=
    \frac{
      \labelStdDev\eta\sqrt{\pi/2}
    }{
      1+\featureRadiusBound^2/(n\lambda)
    }.
  \end{equation}

  Next, we upper bound the term
  $\| \huberLooGrad(\noisyWeights) - \huberLooGrad(\optWeights)\|_2$.
  Using \cref{lem:huber-loo-grad-difference-bounded}, since
  $\noisyWeights - \optWeights = \dpNoise$ by definition
  (see \eqref{eq:output-perturbation}),
  we can write
  \[
    \huberLooGrad(\noisyWeights) - \huberLooGrad(\optWeights) =
    \underbrace{(\restOfDS^\top D\restOfDS+n\lambda I)}_{\Mdp} \dpNoise,
  \]
  where $\opnorm{\Mdp} \leq \opnorm{\restOfDS^\top \restOfDS} +
  \lambda n$.
  Due to \cref{ass:full-dim-attack-assumption}, we can apply Remark 5.40 in
  \cite{vershrandmat} with deviation parameter $t=\sqrt{n+d}$ to get
  \[
    \opnorm{\restOfDS^\top \restOfDS}
    \leq (n-1)\opnorm{\ex[x x^\top]}+C(n+d),
  \]
  with probability at least $1-2e^{-c_1(n+d)}$, for absolute
  constants $C,c_1>0$.
  Then, again by \cref{ass:full-dim-attack-assumption}, we have that
  $\opnorm{\ex[x x^\top]}\leq C$ (see e.g.\ the argument
  in Lemma C.1 of \cite{bombari2024privacy}), which readily gives,
  after enlarging $C$,
  \begin{align}\label{eq:huber-known-gram-upper-bound}
    \opnorm{\Mdp}\leq n\lambda+C(n+d).
  \end{align}
  Theorem 3.1.1 in \cite{vershynin2018high}, applied to the independent
  Gaussian coordinates of $\dpNoise/\dpStdDev$, yields
  $\|\dpNoise\|_2=O(\sqrt d\dpStdDev)$ with probability at least
  $1-2e^{-c_2d}$, for an absolute constant $c_2>0$.
  Taking the intersection with the previous event, whose total
  failure probability is at most $2e^{-c_1(n+d)}+2e^{-c_2d}$, gives
  \begin{equation}
    \label{eq:huber-transformed-privacy-noise-bound}
    \| \huberLooGrad(\noisyWeights) - \huberLooGrad(\optWeights) \|_2
    = \|\Mdp\dpNoise\|_2
    = O \left( \dpStdDev \sqrt d \left( n \lambda + n + d \right) \right).
  \end{equation}

  Then, merging \eqref{eq:huber-signal-threshold} and
  \eqref{eq:huber-transformed-privacy-noise-bound} in
  \eqref{eq:huber-attack-signal-plus-noise}, and using
  \cref{lem:reconstruction-error-snr}, gives, with probability at least
  $1-\eta-2e^{-c_1(n+d)}-2e^{-c_2d}$,
  \begin{align}\label{eq:beforepluggingsigma}
    \frac{\reconstructionError}{\featureRadiusBound^2}
    \leq C_1
    \frac{n^2\left(\lambda+ (1+d/n)\right)^2\dpStdDev^2 d}
    {\featureRadiusBound^2\min\{u^2,\clipConstant^2\}},
  \end{align}
  where we used $|\psi_{\clipConstant}(r)|^2=\min\{r^2,\clipConstant^2\}$
  and $|\targetResidOpt|>u$ on the signal event.

  Recalling $u$ from above and $\dpStdDev$ from
  \cref{prop:output-pert-is-private}:
  \[
    u
    :=
    \frac{
      \labelStdDev\eta\sqrt{\pi/2}
    }{
      1+\featureRadiusBound^2/(n\lambda)
    },
    \qquad
    \dpStdDev
    =
    \frac{\sqrt{2}\featureRadiusBound\clipConstant}
    {\lambda\sqrt{\rho}\numPoints},
  \]
  the bound becomes, for an absolute constant $C_2>0$,
  \[
    \frac{\reconstructionError}{\|\targetPoint\|_2^2}
    \leq
    C_2\frac{\featureRadiusBound^2}{\|\targetPoint\|_2^2}
    \max\left\{1,
      \frac{2\clipConstant^2}{\pi\labelStdDev^2\eta^2}
      \left(1+\frac{\featureRadiusBound^2}{n\lambda}\right)^2
    \right\}
    \left(1+\frac{C}{\lambda}\right)^2
    \left(1+\frac dn\right)^2
    \frac{d}{\rho}.
  \]
  Here we used
  \[
    1+\frac{C(1+d/n)}{\lambda}
    \leq
    \left(1+\frac{C}{\lambda}\right)\left(1+\frac dn\right).
  \]
  By \cref{ass:full-dim-attack-assumption},
  $\|\targetPoint\|_2^2=\Theta(d)$ and
  $\featureRadiusBound^2=\Theta(d)$, so the ratio displayed in the
  bound is $O(1)$. Absorbing it into the absolute constant proves
  \eqref{eq:huber-reconstruction-bound}. A union bound over
  the Gram-matrix event, the privacy-noise event, and the signal
  anti-concentration event gives total failure probability at most
  $\eta+2e^{-c_1(n+d)}+2e^{-c_2d}$.
\end{proof}

\section{Proofs for the extension to low-dimensional data}
\label{sec:lower-bounds-sparse-priors}

\getkeytheorem{low-dim-entropy-thm}

\begin{proof}
  Let $U\in\reals^{d\times s}$ have orthonormal columns spanning
  $\operatorname{Im}(P)$, so that $P=UU^\top$ and $U^\top U=I_s$.
  Set $\targetProj=P\targetPoint$ and $\attackProj=P\attackPoint$,
  both in $\operatorname{Im}(P)\subseteq\reals^d$, and let
  $v=U^\top\targetPoint$ and $\widehat v=U^\top\attackPoint$ be their
  coordinates in $\reals^s$. Thus $\targetProj=Uv$,
  $\attackProj=U\widehat v$,
  and the entropy notation in the theorem means
  $h_s(\targetPoint)=h(v)$.

  Fix $\restOfPairsDS$. As in the proof of
  \cref{theorem:rho-entropy-reconstruction}, the data-processing
  inequality and \cref{lem:zcdp-mut-info}, applied to
  $\targetPair=(\targetPoint,\targetLabel)$, give
  \[
    I(v;\widehat v)
    \leq I(\targetPoint;\attackPoint)
    \leq\rho.
  \]
  Let us apply \cref{lem:squared-l2-rate-dist} with $\sourceRV=v$ and $k=s$,
  to the source $v$ and reconstruction
  $\widehat v$ in $\reals^s$. Since $U$ preserves Euclidean norms,
  we obtain
  \begin{equation}
    \ex\|\targetProj-\attackProj\|_2^2
    =\ex\|v-\widehat v\|_2^2
    \geq
    \frac{s}{2\pi e}
    \exp\left(\frac{2}{s}(h_s(\targetPoint)-\rho)\right).
  \end{equation}
  Letting $\targetProjRest=(I-P)\targetPoint
  \in\operatorname{Im}(P)^\perp$, we have
  \begin{align}
    \|\targetPoint-\attackPoint\|_2^2
    &=\|\targetProj-P\attackPoint\|_2^2
    +\|\targetProjRest-(I-P)\attackPoint\|_2^2 \\
    &\geq\|\targetProj-\attackProj\|_2^2.
  \end{align}
  Taking expectations, dividing by  $\ex\|P\targetPoint\|_2^2$ and using
  $h_s(x_s')=h_s(\targetPoint)+\frac
  s2\log\bigl(s/\ex\|P\targetPoint\|_2^2\bigr)$ (again, via an application of Theorem 9.6.4 in \cite{coverElementsInformationTheory2001}) gives
  \eqref{eq:sparse-reconstruction-lower-bound-relative}.  
\end{proof}

\getkeytheorem{thm:low-dim-huber-reconstruction}

\begin{proof}
  Consider the shorthand $a:=
  \psi_{\clipConstant}
  (\targetPoint^\top\optWeights-\targetLabel)$. Recall that $\dpNoise=\noisyWeights-\optWeights$ is the Gaussian
noise defined in \eqref{eq:output-perturbation} and
$\Mdp=\restOfDS^\top D\restOfDS+n\lambda I$ is the matrix
constructed in the proof of \cref{thm:direct-huber-reconstruction}.
  Projecting both sides of \eqref{eq:huber-attack-signal-plus-noise}
  over $P$ gives
  \begin{equation}
    - P \huberLooGrad(\noisyWeights)
    =
    a P\targetPoint - P\Mdp\dpNoise = a \targetPoint - P\Mdp\dpNoise.
  \end{equation}
  Notice that we have
  \[
    P \Mdp \dpNoise = P \Mdp P\dpNoise + P\Mdp(I-P)\dpNoise = P \Mdp P \dpNoise,
  \]
  where we use that $\restOfDS (I-P) = 0$, due to $x_i \in \mathrm{Im}(P)$ for all $i
  \in [n]$. 
  
  Let $U\in\mathbb R^{d\times s}$ have orthonormal columns
spanning $\operatorname{Im}(P)$, and set $Z_-:=\restOfDS U$.
The rows of $Z_-$ are independent $O(1)$-subGaussian vectors
in $\mathbb R^s$. Applying the concentration argument used
for \eqref{eq:huber-known-gram-upper-bound}, now in dimension
$s$, gives
\[
  \opnorm{\restOfDS^\top\restOfDS}
  =\opnorm{Z_-^\top Z_-}
  \leq C(n+s)
\]
with probability at least $1-2e^{-c_2 n}$.
Consequently,
\[
  \opnorm{P\Mdp}\leq\opnorm{\Mdp}
  \leq n\lambda+C(n+s).
\]
  Then, we have that
  \[
    \| P \Mdp \dpNoise \|_2 \leq \opnorm{P \Mdp}  \| P \dpNoise \|_2
    =  O \left( \left( (1+\lambda) n + s \right) \dpStdDev \sqrt s \right),
  \]
  with probability at least $1 - 2 e^{- c_2 n} - 2 e^{- c_3 s}$,
  where the last equality follows from a similar argument as the one
  used to obtain \eqref{eq:huber-transformed-privacy-noise-bound}.

  Then, the argument follows the same path as the one for the proof
  of \cref{thm:direct-huber-reconstruction}, after
  \eqref{eq:huber-transformed-privacy-noise-bound} and using
  \cref{lem:reconstruction-error-snr}. This yields, with probability at least
  $1- \eta -2 e^{-c_2 n }-2 e^{-c_3 s}$,
  \[
    \frac{\reconstructionError}{\|\targetPoint\|_2^2}
    \leq C_1\frac{\featureRadiusBound^2}{\|\targetPoint\|_2^2}
    \frac{n^2\left(\lambda+ (1+s/n)\right)^2\dpStdDev^2 s}
    {\featureRadiusBound^2\min\{u^2,\clipConstant^2\}},
  \]
  where $u$ is defined as a function of $\eta$ in
  \eqref{eq:huber-signal-threshold}. Apart from the factor
$\featureRadiusBound^2/\|\targetPoint\|_2^2$,
the RHS is the bound in \eqref{eq:beforepluggingsigma}
with $d$ replaced by $s$. Furthermore,
$\|\targetPoint\|_2^2=\Theta(s)$ and
$\featureRadiusBound^2=\Theta(s)$, so that $\featureRadiusBound^2/\|\targetPoint\|_2^2=O(1)$.
Following the same substitution as at the end of the proof
of \cref{thm:direct-huber-reconstruction}, with
$n=\Omega(s)$, $\lambda=\Omega(1)$, and
$\clipConstant=O(\labelStdDev\eta)$, we obtain the desired
result by taking $\eta$ to be a sufficiently small absolute
constant.
\end{proof}

\section{Comparison with \citet{balleReconstructingTrainingData2022}}
\label{app:comparison-bounds}

Corollary~4 in \citep{balleReconstructingTrainingData2022} bounds reconstruction success in terms of the
small-ball probability
$\kappa(r):=\sup_{v\in\reals^d}\prob(\|\targetPoint-v\|_2\leq r)$.
Writing $\bar\kappa(r)\in(0,1]$ for an upper bound on $\kappa(r)$,
we convert their guarantee to expected squared error by integrating
the corresponding tail bound:
\begin{equation}
  \label{eq:balle-expected-error-comparison}
  D:=\ex\|\targetPoint-\attackPoint\|_2^2\geq\int_0^\infty 2r\left[
    1-\exp\left(-\left[
      \sqrt{\log(1/\bar\kappa(r))}-\sqrt{\rho}
    \right]_+^2\right)
  \right]\,dr.
\end{equation}
This follows from
$D=\int_0^\infty 2r\prob(\|\targetPoint-\attackPoint\|_2>r)\,dr$.

For $\targetPoint\sim\N(0,I_d)$, we use the Chernoff bound from
their proof of Proposition~7,
\[
  \bar\kappa(r)=
  \begin{cases}
    \exp\!\left[\dfrac d2\left(1-\dfrac{r^2}{d}
      +\log\dfrac{r^2}{d}\right)\right],&0<r\leq\sqrt d,\\
    1,&r>\sqrt d.
  \end{cases}
\]
For the uniform distribution on the ball of radius $\sqrt d$, their
Proposition~6 proof gives, after rescaling,
$\kappa(r)=\min\{(r/\sqrt d)^d,1\}$.
Thus, our comparison uses the small-ball bounds in their proofs,
retaining their radius dependence before the asymptotic
simplifications in the statements of the propositions.

\section{Experimental details and additional numerical results}
\label{app:experimental-details}
This appendix describes the datasets (Appendix \ref{app:datasets}) and the
experimental procedure (Appendix \ref{app:experimental-process}), and then
provides more details for the figures in the main text. We provide
further visual reconstructions for the full-dimensional attack
(Appendix \ref{app:visual-reconstruction-seeds})
and the remaining experiments on the low-dimensional setting (Appendix \ref{app:pca-reconstruction-rest}). 
Parameter configurations are given in
Appendix \ref{app:figure-hyperparameters}.

\subsection{Datasets}
\label{app:datasets}

\paragraph{Synthetic data.}
We draw $\numPoints=1000$ training examples independently and uniformly
from the sphere of radius $\featureRadiusBound:=\sqrt{\dimInputs}$. We draw
$\trueWeights$ uniformly from the unit sphere and generate labels according to \cref{eq:datamodel}.  Validation
and test sets of $10{,}000$ examples are drawn from the same distribution.

\paragraph{Synthetic rank-$s$ data.}
For the low-dimensional experiments, we design a distribution that
has a subspace with a strong signal, and a weak component on its
complement. For $s<\dimInputs$, each example is initially generated as
$x=\targetProj+\targetProjRest$, where the two components are independent
and uniform on spheres in complementary coordinate blocks of size $s$ and
$\dimInputs-s$, with respective radii $\sqrt{s}$ and
$\bulkRatio\sqrt{s/(1-\bulkRatio^2)}$. Then,
$\ex[\targetProj\targetProj^\top]=I$ on the signal block,
$\ex[\targetProjRest\mid\targetProj]=0$, and every example has norm
$\|x\|_2=\sqrt{s/(1-\bulkRatio^2)}$, with
$\|\targetProjRest\|_2/\|x\|_2=\bulkRatio$.
The experiments in \cref{fig:pca-reconstruction-synthetic} use
$\bulkRatio=0.01$, so the data lies near, rather than exactly in, the
signal subspace.
A fixed Householder reflection is applied to every split, mixing the
coordinate blocks without changing these norms. The attacker estimates
the rank-$s$ projector from the leading right singular vectors of
$\restOfDS$; the true signal subspace is not supplied to the attack.
The feature clipping radius $\featureRadiusBound$ is set $0.01$ above the
sample norm to avoid clipping due to rounding. We use $1000$ training
examples and $2000$ examples in each of the validation and test sets,
with $\trueWeights$ and labels generated as in the preceding paragraph.

\paragraph{CIFAR-10.}
We form a binary frog-versus-truck problem from classes $6$ and $9$.  The
official training data for those classes are split into $8000$ training and
$2000$ validation examples, and the official test split provides $2000$ test
examples.

\paragraph{ImageNet.}
We form binary problems from pairs of ImageNet-21K classes, specifically African
elephant (\texttt{n02504458}) versus giant panda (\texttt{n02510455}).  The
pooled images of each class are shuffled and split $80/10/10$, giving $2160$
training and $270$ validation and test examples. Images are resized
by downscaling to produce data of varying dimension.

\paragraph{Image preprocessing.}
RGB pixel values are mapped to $[0,1]$ and centered as follows: for each dataset and resolution, we subtract a
fixed public mean image, fitted using $10{,}000$ training
images from classes outside the private binary pair. The same mean is reused across the training, validation, and test
splits. Each centered feature vector is then rescaled to norm
$\sqrt{\dimInputs}$, with $\featureRadiusBound=\sqrt{\dimInputs}$.
Images retain their binary class labels, and label noise is not added.

\subsection{Experimental procedure}
\label{app:experimental-process}

For each run, data is generated as described above. The learner's
private linear regression is implemented using a damped Newton solver
on the clipped, regularized Huber objective in
\eqref{eq:clippedminimization}, running until convergence.

The private release is implemented by adding Gaussian noise as in
\eqref{eq:output-perturbation}, with the variance set by
\eqref{eq:output-perturbation-noise-calibration}.

\paragraph{Replicates and aggregation.}
Each quantitative synthetic setting attacks one target for each of three
independent seeds, hence three target attacks in total. Each quantitative image
setting attacks one target from each binary class for each of three seeds,
hence six target attacks in total. Hyperparameters are selected using the
mean validation risk over these replicates. Each plotted heatmap cell or curve point is then the
mean reconstruction error over the same target attacks. The visual
reconstruction panels are not averaged: each displays four targets (two from
each class) for a single random seed. The numbers of seeds and targets per class
used for each figure are listed in
Appendix~\ref{app:figure-hyperparameters}.

\paragraph{Privacy of hyperparameter selection.}

Our experiments compare output-perturbation mechanisms to illustrate the theoretical scaling laws, rather than evaluate an end-to-end private model-selection procedure. We therefore select hyperparameters using validation performance without accounting for the privacy cost of selection. The reported \(\rho\) guarantees apply to each trained mechanism, not the complete process including hyperparameter selection. Naturally, a real-world deployment on sensitive data would require private hyperparameter selection, for example using the methods of \cite{papernot2022hyperparameter, koskela2023practical}, and would give guarantees that covered the entire process.

\paragraph{Image visualisation.}

For display of centered reconstructions, we choose the sign to give a
nonnegative inner product with the centered, normalized target. We then
multiply by the target's saved inverse normalization scale, add back the
public mean image, and clip RGB values to $[0,1]$. 
We use the sign and scale of the target for visualization only.

\newpage
\subsection{Projected reconstruction: remaining figures}
\label{app:pca-reconstruction-rest}

\begin{figure}[H]
\centering
\includegraphics[width=0.55\linewidth]{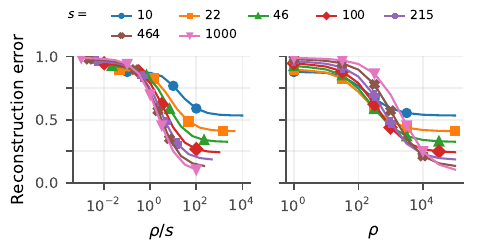}
\caption{Projected  reconstruction error on ImageNet
(elephant vs.\ panda) data with $\dimInputs=12{,}288$. The projected attack
uses a rank-$s$ PCA subspace estimated from the known data. The left panel
plots the projected attacks against $\rho/s$; the right panel plots them against $\rho$. Colours indicate
the projection rank $s$.}
\label{fig:pca-reconstruction-imagenet}
\end{figure}

\subsection{Further visual reconstructions}
\label{app:visual-reconstruction-seeds}
In addition to \cref{fig:visual-reconstruction}, we present several more examples of reconstruction of ImageNet data in  \cref{fig:visual-reconstruction-seeds}; 
and similarly for CIFAR-10 in \cref{fig:cifar10-visual-reconstruction-seeds}.
\begin{figure}[H]
\centering
\includegraphics[width=0.48\linewidth]{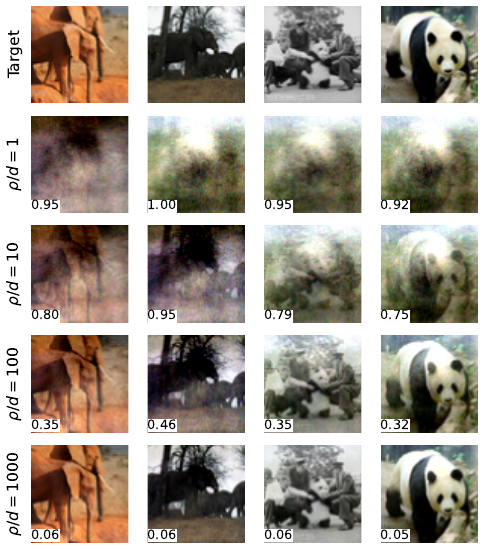}\hfill
\includegraphics[width=0.48\linewidth]{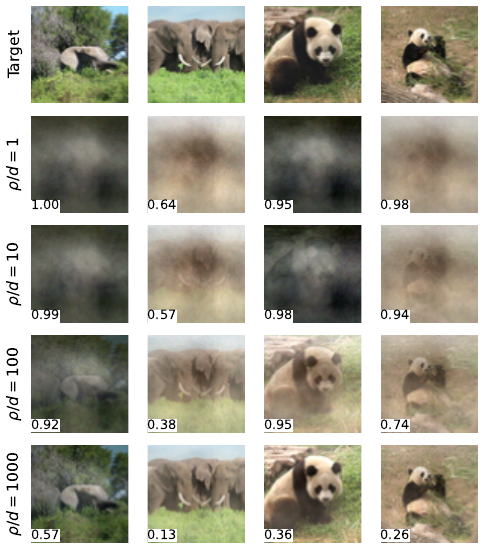}\par
\vspace{-0.5em}
\captionsetup[subfigure]{justification=centering,singlelinecheck=false}
\noindent\makebox[\linewidth][l]{%
\subcaptionbox{\label{fig:visual-imagenet-seed-401}}[0.48\linewidth]{\rule{0pt}{0pt}}\hfill
\subcaptionbox{\label{fig:visual-imagenet-seed-403}}[0.48\linewidth]{\rule{0pt}{0pt}}%
}
\caption{Additional ImageNet elephant-versus-panda reconstructions,
presented as in \cref{fig:visual-reconstruction}.}
\label{fig:visual-reconstruction-seeds}
\end{figure}
\newpage
\begin{figure}[H]
\centering
\includegraphics[width=0.45\linewidth]{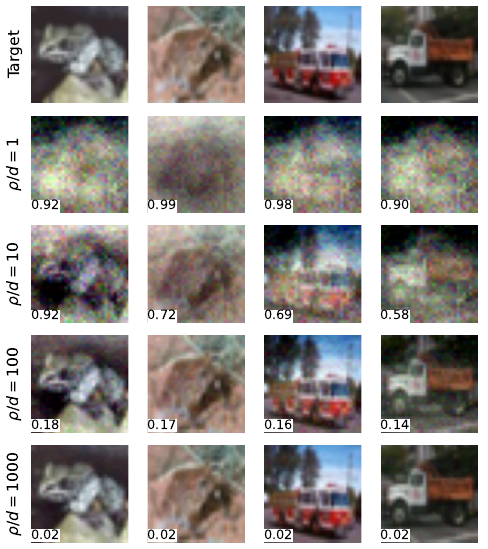}\hfill
\includegraphics[width=0.45\linewidth]{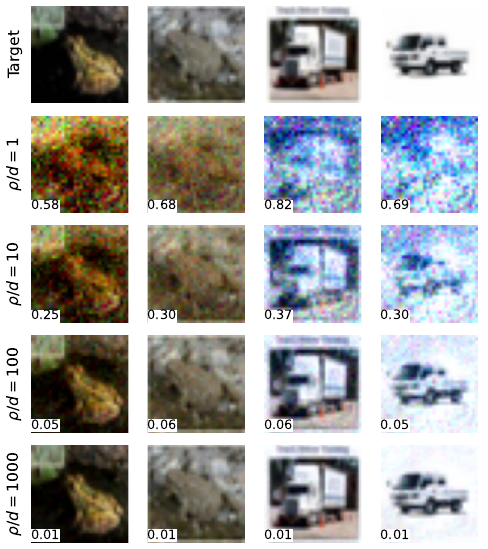}\par
\vspace{-0.5em}
\captionsetup[subfigure]{justification=centering,singlelinecheck=false}
\noindent\makebox[\linewidth][l]{%
\subcaptionbox{\label{fig:visual-cifar10-seed-401}}[0.45\linewidth]{\rule{0pt}{0pt}}\hfill
\subcaptionbox{\label{fig:visual-cifar10-seed-402}}[0.45\linewidth]{\rule{0pt}{0pt}}%
}
\par\medskip
\includegraphics[width=0.45\linewidth]{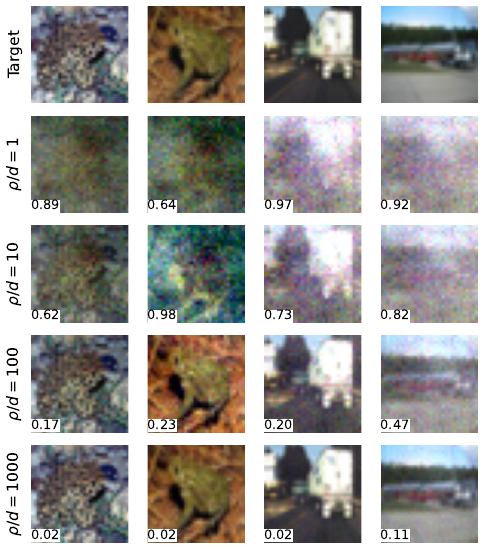}\par
\vspace{-0.5em}
\noindent\makebox[\linewidth][c]{%
\subcaptionbox{\label{fig:visual-cifar10-seed-403}}[0.45\linewidth]{\rule{0pt}{0pt}}%
}
\caption{CIFAR-10 frog-versus-truck reconstructions for all three
target seeds, presented as in \cref{fig:visual-reconstruction}.}
\label{fig:cifar10-visual-reconstruction-seeds}
\end{figure}

\newpage

\subsection{Run configurations per figure}
\label{app:figure-hyperparameters}

Each table below lists the configurations for each figure.
If a parameter had a large number of values, we summarize this as
the endpoints of the range, the number of values, and the increment, if it was a geometric range.

\begin{table}[H]
  \centering
  \footnotesize
  \begin{tabular}{ll}
    \toprule
    Setting & Synthetic \\
    \midrule
    Training examples $\numPoints$ & $1{,}000$ \\
    Validation examples & $10{,}000$ \\
    Test examples & $10{,}000$ \\
    Dimension $\dimInputs$ & $10$--$10{,}000$ (7) \\
    Feature norm $\|x\|_2$ & $\sqrt{\dimInputs}$ \\
    Feature clipping $\featureRadiusBound$ & $\sqrt{\dimInputs}$ \\
    Label noise $\labelStdDev$ & $0.5$ \\
    Privacy budget $\rho$ & $0.1$--$10^{5}$ (7, $\times 10$) \\
    Regularization $\lambda$ & $10^{-4}$--$10$ (11, $\times \sqrt{10}$) \\
    Residual clipping $\clipConstant$ & $10^{-3}$--$1{,}000$ (13, $\times \sqrt{10}$) \\
    \# of seeds & 3 \\
    \bottomrule
  \end{tabular}
  \caption{Configuration for \Cref{fig:combined-synth-viz}.}
  \label{tab:combined-synth-viz}
\end{table}

\begin{table}[H]
  \centering
  \footnotesize
  \begin{tabular}{ll}
    \toprule
    Setting & ImageNet (elephant/panda) \\
    \midrule
    Training examples $\numPoints$ & $2{,}160$ \\
    Validation examples & $270$ \\
    Test examples & $270$ \\
    Dimension $\dimInputs$ & $12$--$12{,}288$ (6, $\times 4$) \\
    Feature norm $\|x\|_2$ & $\sqrt{\dimInputs}$ \\
    Feature clipping $\featureRadiusBound$ & $\sqrt{\dimInputs}$ \\
    Privacy budget $\rho$ & $1$--$10^{5}$ (11, $\times \sqrt{10}$) \\
    Regularization $\lambda$ & $10^{-4}$--$1$ (9, $\times \sqrt{10}$) \\
    Residual clipping $\clipConstant$ & $10^{-3}$--$10$ (9, $\times \sqrt{10}$) \\
    \# of seeds & 3 \\
    Targets per class & 1 \\
    \bottomrule
  \end{tabular}
  \caption{Configuration for \Cref{fig:imagenet}. Hyperparameter selection used validation risk within $\lambda \geq 0.1$ and $\clipConstant \leq 0.01$.}
  \label{tab:imagenet}
\end{table}

\begin{table}[H]
  \centering
  \footnotesize
  \begin{tabular}{ll}
    \toprule
    Setting & ImageNet (elephant/panda) \\
    \midrule
    Training examples $\numPoints$ & $2{,}160$ \\
    Validation examples & $270$ \\
    Test examples & $270$ \\
    Dimension $\dimInputs$ & $12{,}288$ \\
    Feature norm $\|x\|_2$ & $\sqrt{\dimInputs}$ \\
    Feature clipping $\featureRadiusBound$ & $\sqrt{\dimInputs}$ \\
    Privacy budget $\rho$ & $1.23 \times 10^{3}$--$1.23 \times 10^{8}$ (6, $\times 10$) \\
    Regularization $\lambda$ & $1$ \\
    Residual clipping $\clipConstant$ & $0.125$ \\
    \# of seeds & 1 \\
    Targets per class & 2 \\
    \bottomrule
  \end{tabular}
  \caption{Configuration for \Cref{fig:visual-reconstruction}.}
  \label{tab:visual-reconstruction}
\end{table}

\begin{table}[H]
  \centering
  \footnotesize
  \begin{tabular}{lll}
    \toprule
    Setting & Synthetic rank-$s$ signal & CIFAR-10 (frog/truck) \\
    \midrule
    Training examples $\numPoints$ & $1{,}000$ & $8{,}000$ \\
    Validation examples & $2{,}000$ & $2{,}000$ \\
    Test examples & $2{,}000$ & $2{,}000$ \\
    Dimension $\dimInputs$ & $1{,}000$ & $3{,}072$ \\
    Feature norm $\|x\|_2$ & --- & $\sqrt{\dimInputs}$ \\
    Feature clipping $\featureRadiusBound$ & $3.17$--$31.6$ (7) & $\sqrt{\dimInputs}$ \\
    Label noise $\labelStdDev$ & $0.1$ & --- \\
    Privacy budget $\rho$ & $10$--$10{,}000$ (7, $\times \sqrt{10}$) & $10$--$10{,}000$ (10, $\times 2.15$) \\
    Regularization $\lambda$ & $10^{-4}$--$1$ (9, $\times \sqrt{10}$) & $3.07, 307, 3{,}072$ \\
    Residual clipping $\clipConstant$ & $10^{-4}$--$1$ (9, $\times \sqrt{10}$) & $0.1$ \\
    Attacker PCA rank $s$ & $10$--$1{,}000$ (7) & $10$--$1{,}000$ (7) \\
    Signal rank of the data & $10$--$1{,}000$ (7) & --- \\
    Bulk-to-signal ratio $\bulkRatio$ & $0.01$ & --- \\
    \# of seeds & 3 & 3 \\
    Targets per class & --- & 1 \\
    \bottomrule
  \end{tabular}
  \caption{Configuration for \Cref{fig:pca-reconstruction}. Hyperparameter selection used validation risk within $\lambda \geq 0.1$ and $\clipConstant \leq 0.1\labelStdDev$ for Synthetic rank-$s$ signal; $\lambda \geq 10$ and $\clipConstant \leq 0.1$ for CIFAR-10 (frog/truck).}
  \label{tab:pca-reconstruction}
\end{table}

\begin{table}[H]
  \centering
  \footnotesize
  \begin{tabular}{ll}
    \toprule
    Setting & ImageNet (elephant/panda) \\
    \midrule
    Training examples $\numPoints$ & $2{,}160$ \\
    Validation examples & $270$ \\
    Test examples & $270$ \\
    Dimension $\dimInputs$ & $12{,}288$ \\
    Feature norm $\|x\|_2$ & $\sqrt{\dimInputs}$ \\
    Feature clipping $\featureRadiusBound$ & $\sqrt{\dimInputs}$ \\
    Privacy budget $\rho$ & $1$--$10^{5}$ (11, $\times \sqrt{10}$) \\
    Regularization $\lambda$ & $1$--$1{,}000$ (7, $\times \sqrt{10}$) \\
    Residual clipping $\clipConstant$ & $0.01, 0.1$ \\
    Attacker PCA rank $s$ & $10$--$1{,}000$ (7) \\
    \# of seeds & 3 \\
    Targets per class & 1 \\
    \bottomrule
  \end{tabular}
  \caption{Configuration for \Cref{fig:pca-reconstruction-imagenet}. Hyperparameter selection used validation risk within $\lambda \geq 10$ and $\clipConstant \leq 0.01$.}
  \label{tab:pca-reconstruction-imagenet}
\end{table}

\begin{table}[H]
  \centering
  \footnotesize
  \begin{tabular}{ll}
    \toprule
    Setting & ImageNet (elephant/panda) \\
    \midrule
    Training examples $\numPoints$ & $2{,}160$ \\
    Validation examples & $270$ \\
    Test examples & $270$ \\
    Dimension $\dimInputs$ & $12{,}288$ \\
    Feature norm $\|x\|_2$ & $\sqrt{\dimInputs}$ \\
    Feature clipping $\featureRadiusBound$ & $\sqrt{\dimInputs}$ \\
    Privacy budget $\rho$ & $1.23 \times 10^{3}$--$1.23 \times 10^{8}$ (6, $\times 10$) \\
    Regularization $\lambda$ & $1$ \\
    Residual clipping $\clipConstant$ & $0.125$ \\
    \# of seeds & 2 \\
    Targets per class & 2 \\
    \bottomrule
  \end{tabular}
  \caption{Configuration for \Cref{fig:visual-reconstruction-seeds}.}
  \label{tab:visual-reconstruction-seeds}
\end{table}

\begin{table}[H]
  \centering
  \footnotesize
  \begin{tabular}{ll}
    \toprule
    Setting & CIFAR-10 (frog/truck) \\
    \midrule
    Training examples $\numPoints$ & $8{,}000$ \\
    Validation examples & $2{,}000$ \\
    Test examples & $2{,}000$ \\
    Dimension $\dimInputs$ & $3{,}072$ \\
    Feature norm $\|x\|_2$ & $\sqrt{\dimInputs}$ \\
    Feature clipping $\featureRadiusBound$ & $\sqrt{\dimInputs}$ \\
    Privacy budget $\rho$ & $307$--$3.07 \times 10^{7}$ (6, $\times 10$) \\
    Regularization $\lambda$ & $1$ \\
    Residual clipping $\clipConstant$ & $0.125$ \\
    \# of seeds & 3 \\
    Targets per class & 2 \\
    \bottomrule
  \end{tabular}
  \caption{Configuration for \Cref{fig:cifar10-visual-reconstruction-seeds}.}
  \label{tab:cifar10-visual-reconstruction-seeds}
\end{table}

\end{document}